\documentclass[11pt]{article}
\usepackage[margin=1in]{geometry}
\usepackage{amsmath, amssymb, amsthm, bm}
\usepackage{graphicx}
\usepackage{comment}
\usepackage{xcolor}
\usepackage{natbib}
\usepackage{amsmath,amssymb}
\usepackage{booktabs} 
\newtheorem{lemma}{Lemma}
\newtheorem{theorem}{Theorem}
\title{On Stability and Decomposition of Sample Quantiles under Heavy-Tailed Distributions}
\author{Choudur Lakshminarayan \\
School of Business, Stevens Institute of Technology\\ email: lchoudur@stevens.edu}
\date{}

\begin{document}

\maketitle

\begin{abstract}
We study sample quantiles for Value-at-Risk tail measures based on linear projections of financial returns whose underlying distributions are heavy-tailed, when both the projection direction and the empirical quantile threshold are estimated from the data. Bahadur's representation is the canonical starting point for sample quantile asymptotics, but under a fixed distribution it identifies only the empirical quantile fluctuation and does not by itself separate the instability induced by estimating the projection direction.  Empirical-process theory provides useful scaffolding through half-spaces, symmetric differences, and Glivenko--Cantelli uniform convergence. These tools yield stability bounds, but they absorb changes in projection direction and changes in quantile threshold into a single symmetric-difference measure. Thus, a global uniform-convergence requirement is imposed on what is intrinsically a local quantile-stability problem.

This paper introduces a Q-Q orthogonality formulation that separates projection-direction and quantile-threshold effects. The object of interest is $\hat q_\alpha(\hat w)-q_\alpha(w_0)$, the difference between the empirical quantile computed using the estimated projection direction and the population quantile computed at the reference direction. We decompose this difference as $\hat q_\alpha(\hat w)-q_\alpha(w_0)=D_1+D_2+D_3$, where $D_1$ measures population quantile movement induced by perturbing the projection direction, $D_2$ measures empirical quantile fluctuation with the projection direction held fixed, and $D_3$ is the Bahadur-type remainder.
\end{abstract}

\section{Notation}

\begin{itemize}

\item{$R \equiv (R_{1}, R_{2}, \dots, R_{p}) \in \mathbb{R}^{p}$: Return vector}

\item{$w \equiv (w_{1}, w_{2}, \dots, w_{p}) \in \mathbb{R}^{p}$: Weight vector}

\item{$\hat{w}$: Estimated weight vector}

\item{$L = -w^{\top} R$: Loss, a scalar function of returns}

\item{$L_{0} = -w_{0}^{\top} R$: Reference loss at $w_{0}$}

\item{$F_{L}(t) = \mathbb{P}(L \le t)$: Distribution function of $L$}

\item{$F_{n}(t) = \frac{1}{n}\sum_{i=1}^{n} \mathbf{1}\{L_{i} \le t\}$: Empirical distribution function}

\item{$q_{\alpha}$, $q_{\alpha}(w_{0})$: Population $\alpha$-quantile satisfying $F_{L}(q_{\alpha}) = \alpha$}

\item{$VaR_{\alpha}$= True VaR at level $\alpha$}.

\item $\widehat{\mathrm{VaR}}_\alpha$: Estimate of VaR at level $\alpha$

\item{$\hat{q}_{\alpha}$: Sample quantile based on $F_{n}$}

\item{$\hat{q}_{\alpha}(\hat{w})$: Sample quantile indexed by estimated weights}

\item{$\mathrm{VaR}_{\alpha}(L) = \inf \left\{ \ell \in \mathbb{R} : F_L(\ell) \ge \alpha \right\}$: Value at Risk}

\item{$A \equiv \{-\hat{w}^{\top} R < t\}$: Sub-level set generated by threshold $t$}

\item{$B \equiv \{-w_{0}^{\top} R < q_{0}\}$: Sub-level set generated by reference quantile $q_{0}$}

\item{$I_{A}, I_{B}$: Indicator functions of sets $A$ and $B$}

\item{$A \Delta B$: Symmetric difference between sets $A$ and $B$}
\item{$L_1(P)=\left\{f:\int_{\Omega}|f(\omega)|\,dP(\omega)<\infty\right\}$}
\item{$\| \cdot \|$: Euclidean norm}

\item{$\mathbb{E}\|R\|$: First moment of the return vector}

\item{$D_{1}, D_{2}, D_{3}$: Components of the Q-Q orthogonality decomposition}

\item{$R_{n}$: Bahadur remainder term}

\end{itemize}
\section{Introduction} \label{Intro}

Quantile estimation is central in many applications. In finance, portfolio choice is classically formulated through weighted returns in the mean--variance framework of  \cite{markowitz1952}, while Value-at-Risk (VaR) \ \cite{jorion2006}  is a tail-dependent risk measure used to quantify potential losses at a prescribed confidence level $\alpha$. Let $L=-w^\top R$ denote the portfolio loss associated with return vector $R$ and weight vector $w$. Then $\mathrm{VaR}_\alpha(L)$ is the threshold exceeded only with probability $1-\alpha$. Closely related is Conditional Value-at-Risk (CVaR) \cite{RockafellarUryasev2000}, defined as the expected loss beyond the VaR threshold.

Financial returns are well known to exhibit heavy-tailed behavior, making Gaussian approximations unreliable for tail-risk estimation. This motivates the study of stable estimation procedures for VaR under heavy-tailed models, including the multivariate $t$-distribution \cite{Peiro1994}. Formally, for $\alpha \in \{0.90,0.95,0.99\}$, $\mathrm{VaR}_{\alpha}(L)=\inf\{\ell\in\mathbb{R}:F_L(\ell)\geq \alpha\}$, where $F_L$ denotes the distribution function of $L$. Thus, $\mathrm{VaR}_{\alpha}(L)$ is the $\alpha$-quantile $q_\alpha$ of the loss distribution.

A classical starting point for sample quantile asymptotics is  representation by \cite{bahadur1966}. For a fixed distribution, the sample quantile satisfies $\hat q_\alpha-q_\alpha=\{\alpha-F_n(q_\alpha)\}/f(q_\alpha)+R_n$, where $f(q_\alpha)$ is the density evaluated at the population quantile and $R_n$ is a smaller-order remainder term. By viewing quantile estimation through the empirical process generated by Bernoulli indicators $\mathbf{1}(X_i\leq x)$, Bahadur showed that the leading fluctuation of the sample quantile is governed by the empirical distribution function evaluated at $q_\alpha$, while the remainder is asymptotically negligible. This representation forms the basis for much of modern quantile asymptotics and asymptotic statistical analysis \cite{vandervaart1998}.

Modern empirical-process theory provides a natural framework for studying the stability of projected quantiles through classes of half-spaces and indicator functions indexed by projection directions and thresholds \cite{pollard1984}, \cite{vandervaartwellner2023},\cite{dudley1999}. In particular, the class of half-space indicators generated by pairs $(w,t)$ forms a VC class, permitting Glivenko--Cantelli uniform convergence of empirical and population probabilities. These ideas underlie asymptotic analyses of quantiles, empirical measures, and stochastic fluctuations indexed by function classes.

In portfolio applications, however, the loss distribution is itself indexed by the projection direction $w$. As the weights vary, the projected loss $L=-w^\top R$ induces a continuum of projected distributions and corresponding quantiles. Consequently, the empirical quantity $\hat q_\alpha(\hat w)$ contains several sources of variation simultaneously: perturbations in the projection direction, movement of the quantile threshold, and residual empirical fluctuation. The half-space approach yields useful stability bounds through probabilities of symmetric differences $P(A\Delta B)$ between projected loss regions. However, these bounds aggregate the effects of perturbations in the projection direction and changes in the quantile threshold into a single discrepancy measure. More fundamentally, uniform convergence over a continuum of projection directions and thresholds imposes a global empirical-process requirement on what is intrinsically a local quantile-stability problem.

We propose a Q-Q orthogonality formulation to analyze local stability of sample quantiles 
 in the neighborhood of the reference quantile to separate the projection direction effect, the sample quantile threshold effect, and the remainder. Building on Bahadur's representation and the empirical-process framework for half-spaces, we derive a decomposition of the total quantile error into components associated with perturbations in the projection direction, empirical quantile fluctuation with the direction held fixed, and the Bahadur-type remainder. The resulting formulation refines the aggregate half-space bound into interpretable components while retaining the classical asymptotic structure under heavy-tailed multivariate $t$-models.
\section{Our Contribution} \label{Contribution}

Typically, the returns $R \equiv(R_1,R_2,\dots,R_p)$ are modeled by symmetric heavy-tailed families of distributions, which complicates the estimation of VaR. As discussed in the Introduction (Section~\ref{Intro}), VaR is a quantile of the distribution of the loss $L$, which depends on weights estimated from the sample of returns in the  framework of \cite{markowitz1952}. Typically, the sample quantile is obtained from a fixed distribution satisfying $F_X(x_\alpha)=\alpha$. In our setting, however, the sample quantile $\hat q_\alpha(\hat w)$ is indexed by an estimated weight vector and therefore induces additional variability.

Empirical Process Theory provides a natural framework for deriving bounds and evaluating the stability of estimated sample quantiles \cite{vapnik1971uniform,pollard1984,dudley1999}. Using half-spaces and symmetric differences of half-space sub-level sets, $A=\{w^\top R<t\}$ and $B=\{w_0^\top R<q_0\}$, we derive bounds on the symmetric difference probability $P(A\triangle B)$. Under a local Lipschitz condition near the reference quantile, these bounds yield quantile stability by ensuring that small perturbations in $(w,t)$ correspond to small changes in half-space probabilities.  However, the resulting bound is aggregated in nature. Even when sharpened under multivariate $t$-models, it does not distinguish variability arising from changes in the projection direction $w$ from variability arising through shifts in the quantile threshold $t$. Moreover, the empirical-process machinery underlying Glivenko--Cantelli uniform convergence is most naturally formulated under i.i.d. sampling assumptions, whereas financial returns are often serially dependent and only approximately stationary in practice. More fundamentally, bounding over a continuum of half-spaces transforms what is intrinsically a local quantile-stability phenomenon into a global uniform-convergence problem.

Our object of interest is the fluctuation of the sample quantile $\hat q_\alpha(\hat w)$ in a neighborhood of the reference quantile $q_\alpha$. Viewing the problem locally leads naturally to the Q-Q orthogonality decomposition into the components $D_1$, $D_2$, and $D_3$, which separately capture variability induced by changes in the projection direction, empirical quantile fluctuation with the direction held fixed, and the higher-order Bahadur remainder, respectively.

\section{Preliminaries and Assumptions} \label{Prelims}

In studying the behavior of the sample quantile near its population counterpart, consider a projection vector $w\in\mathbb R^p$ and the scalar random variable $L=-w^\top R$. For any threshold $t\in\mathbb R$, the sublevel set $\{-w^\top R<t\}$ corresponds to the interval $(-\infty,t)$ with probability $F_L(t)=P(-w^\top R<t)$, where $F_L(\cdot)$ denotes the distribution function of the loss $L$. Let $q_\alpha$ denote the $\alpha^{\text{th}}$ population quantile. A quantile may be viewed as a first passage time of the process $F_L(t)=\alpha$, where the solution $t$ is the crossing point $t=q_\alpha$. Its empirical counterpart $\hat q_\alpha$ is defined analogously through the empirical distribution.

Our interest lies in analyzing the stability of $t=\hat q_\alpha$ in a neighborhood of $q_\alpha$. For a fixed projection direction $w$, we regard $t$ as a local threshold fluctuating near $q_\alpha(w)$. This enables us to study the behavior of the sample quantile through perturbations of $F(w,t)$ around $F(w,q_\alpha(w))=\alpha$. Variations in $t$ correspond to pointwise movements along the real line, and the local behavior is governed by the density $f_L(q_\alpha(w))$ bounded away from $0$.

This univariate reduction, however, is only valid pointwise in $w$. In our setting the weights are estimated, so each realization $\hat w$ induces a distinct projection $L_{\hat w}=-\hat w^\top R$ and a corresponding quantile $q_\alpha(\hat w)$. Accordingly, the threshold $t$ is implicitly tied to the projection direction $w$ through the relation $F_L(t)=\alpha$. As $w$ varies, both the induced distribution  of $L$ and the associated quantile $q_\alpha(w)$ vary. Consequently, the collection of pairs $(w,t)$ indexes a family of sublevel sets across different projections rather than intervals arising from a single univariate distribution.

Stability of $\hat q_\alpha(\hat w)$ therefore cannot be reduced to interval probabilities of a single random variable. Instead, it requires analyzing both perturbations in $t$ for fixed $w$ and shifts in the distribution induced by variation in $w$. These effects motivate the Q-Q orthogonality decomposition. To formalize the asymptotic framework developed in the sequel, we impose the following assumptions on the return process and projected loss distribution.

\begin{itemize}
\item Let $R\in\mathbb R^p$ denote returns with distribution $P$.
\item $R_1,\dots,R_n$ are independent copies of a random vector $R\in\mathbb R^p$ with common distribution $P$.
\item $R\sim t_\nu(\mu,\Sigma)$ with $\nu>2$.
\item $\mathbb E\|R\|<\infty$.
\item The density of $w^\top R$ is bounded near $q_\alpha$.
\end{itemize}

Furthermore, define $F(w,q)$ as the probability that the projected loss $-w^\top R$ is less than or equal to $q$. For fixed $w$, this is simply the cumulative distribution function of the projected loss $L_w=-w^\top R$ evaluated at the threshold $q$. Although the population symmetric-difference bound requires only $\nu>1$, we impose $\nu>2$ in the multivariate $t_\nu$ framework developed here. This ensures the existence of projected first and second moments and allows $\Sigma$, up to the standard scaling, to be interpreted as the covariance matrix.

\subsection{Empirical Process Foundations}

Empirical process theory provides the natural framework for quantifying the difference between empirical quantities and their population counterparts. In the present setting, the relevant fluctuations arise from indexed half-space probabilities and their associated empirical processes. We follow the standard empirical-process notation and viewpoint developed in \cite{pollard1984},  \cite{dudley1999},  \cite{vandervaartwellner2023}, and  \cite{vandervaart1998}. Fundamentally, the remit of empirical process theory is to study the stochastic fluctuations of $(P_n-P)f$ over classes of functions $f\in\mathcal F$.

In quantile analysis, the relevant class consists of indicator functions of half-spaces of the form $\{-w^\top R<t\}$. For each pair $(w,t)$, let $g_{w,t}$ denote the indicator of the half-space $\{-w^\top r<t\}$. Define $F_n(w,t)$ as the empirical half-space probability and $F(w,t)$ as the corresponding population probability. The empirical-process fluctuation is then $F_n(w,t)-F(w,t)$, equivalently $(P_n-P)g_{w,t}$.

Since both $w$ and $t$ vary in the quantile problem, a single fixed half-space does not capture the full geometry. The indexed family of half-spaces supplies useful empirical-process scaffolding, while the population perturbation analysis below keeps track of how changes in $w$ and $t$ affect the corresponding probabilities. To formalize this setting, consider the indexed class of indicator functions $\mathcal F=\{\mathbf 1\{-w^\top r<t\}:w\in\mathbb R^p,\ t\in\mathbb R\}$. Since these are indicator functions of half-spaces, the class is uniformly bounded and forms a VC class, placing it within the standard Glivenko--Cantelli framework for empirical measures.

\subsection{Pointwise Law of Large Numbers}

Let $\{\mathbf{R_i}\}_{i=1}^n$ denote a strictly stationary and ergodic sequence of returns. For each fixed pair $(w,t)$, the variables $\mathbf 1\{-w^\top R_i\le t\}$ form a stationary ergodic sequence with mean $F(t;w)=P(-w^\top R\le t)$. Invoking the ergodic theorem \cite{billingsley1995}, $F_n(t;w)\to F(t;w)$ almost surely for each fixed $(w,t)$. In the special case where the observations are i.i.d., this reduces to the classical Strong Law of Large Numbers. However, such pointwise convergence does not by itself give a global law over the continuum of pairs $(w,t)$ since the supremum over an indexed class cannot, in general be interchanged with the limit.

\subsection{Glivenko--Cantelli Background and Local Quantile Stability}

The classical Glivenko--Cantelli framework provides a method for showing that empirical half-space probabilities approximate their population counterparts uniformly over an indexed class. In the i.i.d. setting used for the formal empirical-process comparison, the class $\mathcal F=\{\mathbf 1\{-w^\top r\le t\}:w\in\mathbb R^p,\ t\in\mathbb R\}$ is a VC class and hence is Glivenko--Cantelli. Therefore, $\sup_{(w,t)}|F_n(t;w)-F(t;w)|\to0$ almost surely; see  \cite{pollard1984}.

For weakly dependent financial returns, the corresponding global uniform law requires additional dependence conditions, such as mixing assumptions, and does not follow from ergodicity alone; see Doukhan (1994) and Andrews (1988). Hence we apply the Glivenko--Cantelli framework for establishing uniform approximation. The quantile problem studied here is local: we analyze the behavior of $\hat q_\alpha(\hat w)$ near the reference quantile $q_\alpha$, as 
 the projection direction fluctuates. The half-space approach provides a mechanism for bounding symmetric difference of sub-level sets, while the Q-Q orthogonality decomposition separates the local effects of ($w$, $t$) and the Bahadur-type remainder.
\section{Main Result and Organization of the Paper}
\label{Main_Results}

We now state a stability result that summarizes the main mechanism of the paper. The full theorem, including the local oscillation and uniform Bahadur remainder conditions, is deferred until Section~\ref{Main_Theorem}, after the empirical-process and Q-Q decomposition machinery has been developed. The result establishes stability of the estimated Value-at-Risk when both the projection direction and the empirical quantile are estimated from financial returns following heavy-tailed probability laws.

\begin{theorem}[Stability statement]
\label{EarlyStabilityStatement}

Let $w_0\in\mathbb R^p$ be a fixed reference weight vector, and let $q_0:=q_\alpha(w_0)$ denote the population $\alpha$-quantile of the projected loss $L_0:=-w_0^\top R$. Let $\hat w$ be an estimator of $w_0$, and let $\hat q_\alpha(\hat w)$ denote the empirical $\alpha$-quantile associated with the perturbed direction $\hat w$.

Suppose, for the formal empirical-process argument, that $R_1,\ldots,R_n$ are i.i.d.\ random vectors with common distribution $P$; the indexed half-space class satisfies Glivenko--Cantelli uniform convergence; the projected density remains positive and continuous near the target quantile; the quantile function $w\mapsto q_\alpha(w)$ is continuous near $w_0$; a Bahadur representation holds at the perturbed direction $\hat w$; and $\hat w\xrightarrow{p}w_0$.

Then the estimated quantile admits the decomposition
$
\hat q_\alpha(\hat w)-q_\alpha(w_0)
=
D_1+D_2+D_3,
$
where
$
D_1=q_\alpha(\hat w)-q_\alpha(w_0),
$
$
D_2=
\{\alpha-F_n(\hat w,q_\alpha(\hat w))\}/f_{\hat w}(q_\alpha(\hat w)),
$
and
$
D_3=R_n,
$
with $R_n$ denoting the Bahadur remainder term. Moreover,
$
\hat q_\alpha(\hat w)-q_\alpha(w_0)\xrightarrow{p}0.
$
Equivalently,
$
\widehat{\mathrm{VaR}}_\alpha-\mathrm{VaR}_\alpha\xrightarrow{p}0.
$

\end{theorem}

The theorem shows that stability of the estimated quantile arises from three distinct mechanisms operating simultaneously: Lipschitz continuity of perturbed half-spaces, uniform empirical convergence of the indexed empirical distribution function, and local Bahadur linearization of the empirical quantile. The Q-Q orthogonality decomposition isolates these contributions into the components $D_1$, $D_2$, and $D_3$.

The remainder of the paper builds the scaffolding on which the theorem rests. Section~\ref{Half_Spaces} develops Lipschitz continuity bounds for perturbed half-spaces through symmetric-difference control, while Section~\ref{Mult_t} restricts our approaches to variable projections under multivariate $t_\nu$ distribution models. Section~\ref{GC} studies Glivenko--Cantelli uniform convergence and empirical quantile stability for the indexed half-space class. Section~\ref{QQ} introduces the Q-Q orthogonality decomposition and separates variability induced by perturbations in the projection direction, empirical quantile fluctuation, and the Bahadur remainder. Finally, Section~\ref{Main_Theorem} combines these ingredients to establish the proof of the main stability theorem.  

\section{Symmetric Difference Bounds and Half-Spaces Theory}
\label{Half_Spaces}

To establish the asymptotic behavior of quantile estimators, we study the sensitivity of the probability measure $P(A)=P(-w^\top R<t)$ to changes in the parameters $(w,t)$. Geometrically, the set $A=\{-w^\top R<t\}$ defines a perturbed half-space, while the reference half-space is $B=\{-w_0^\top R<q_0\}$, where $q_0=q_\alpha(w_0)$. This sensitivity is quantified by the probability $P(A\triangle B)$ of the symmetric difference between the two half-spaces $A$ and $B$. Measuring this probability mass describes the local sensitivity of the associated tail events to changes in the projection direction and quantile threshold. The following lemma formalizes this sensitivity through symmetric differences of half-space indicator functions.

\begin{lemma}[Probability of Symmetric Differences, Lipschitz, and Bounds]
\label{HalfSpaceLipschitz}
Let $(w,\hat w)\in\mathbb R^p$ and $(q,t)\in\mathbb R$. Define the half-spaces $A=\{r:-w^\top r\le t\}$ and $B=\{r:-\hat w^\top r\le q\}$. Let $R$ be a random vector with $\mathbb E\|R\|<\infty$, and define $L=-w^\top R$. Assume the distribution of $L$ is locally Lipschitz at $q$, i.e., there exists $C>0$ such that for sufficiently small $\epsilon>0$, $P(|L-q|\le\epsilon)\le C\epsilon$. Then for sufficiently small $\|w-\hat w\|$ and $|t-q|$, there exists $C'>0$ such that
\begin{equation} \label{AgBnd}
P(A\Delta B)\le C'\sqrt{\|w-\hat w\|\mathbb E\|R\|+|t-q|}.
\end{equation}
\end{lemma}

\begin{proof}
Let $L=-w^\top R$ and $\hat L=-\hat w^\top R$. By Lemma~\ref{lem:appendix_SymDiff} in Appendix~A, $P(A\triangle B)=\mathbb E|1_A(R)-1_B(R)|=\|1_A-1_B\|_{L_1(P)}$. The problem therefore reduces to bounding the difference between the perturbed and reference half-space indicators in the $L_1(P)$ metric.

If $r\in A\triangle B$, then the half-space indicators disagree. Hence the difference can occur only when the reference loss lies within the perturbation induced by the weight and threshold changes, implying that $|-w^\top r-q|\le |(\hat w-w)^\top r|+|t-q|$. By Cauchy--Schwarz, $|(\hat w-w)^\top r|\le \|w-\hat w\|\|r\|$. Letting $\eta=\|w-\hat w\|\|R\|+|t-q|$, we obtain $A\Delta B\subseteq\{|L-q|\le\eta\}$, so $P(A\Delta B)\le P(|L-q|\le\eta)$.

The local Lipschitz condition applies to deterministic neighborhoods near $q$, whereas the interval width $\eta$ is random because it depends on the return vector $R$. To account for this randomness, fix $\epsilon>0$ and split according to whether the perturbation width exceeds this deterministic threshold. This gives $P(|L-q|\le\eta)\le P(|L-q|\le\epsilon)+P(\eta>\epsilon)$. The first term is handled by local Lipschitz continuity, so $P(|L-q|\le\epsilon)\le C\epsilon$ for sufficiently small $\epsilon$. The second term is bounded by Markov's inequality, $P(\eta>\epsilon)\le \mathbb E\eta/\epsilon$. Since $\mathbb E\eta\le \|w-\hat w\|\mathbb E\|R\|+|t-q|$, it follows that $P(A\triangle B)\le C\epsilon+\{\|w-\hat w\|\mathbb E\|R\|+|t-q|\}/\epsilon$. Choosing $\epsilon=\{\|w-\hat w\|\mathbb E\|R\|+|t-q|\}^{1/2}$ yields the claimed bound, up to a change in the constant.

Thus, for $\|w-\hat w\|\mathbb E\|R\|+|t-q|$ sufficiently small, $P(A\Delta B)\le C'\sqrt{\|w-\hat w\|\mathbb E\|R\|+|t-q|}$, where $C'$ absorbs the local Lipschitz and Markov constants.
\end{proof}

\subsection{Interpreting Symmetric Difference Bounds}

By Lemma~\ref{HalfSpaceLipschitz}, the probability of the symmetric difference between the sets $A$ and $B$ satisfies $\|1_A-1_B\|_{L_1(P)}=P(A\triangle B)$. Consequently, inequality~\eqref{AgBnd} provides an aggregate bound on deviations through symmetric differences of half-spaces. The bound is useful because it shows that small perturbations in the projection direction and the threshold induce only small changes in half-space probability. However, it does not distinguish variability arising from changes in the weights from variability arising through shifts in the quantile.
Perturbing $w$ changes the projected loss $L(w)=-w^\top R$, while perturbing $t$ changes the threshold along a fixed projected direction. The symmetric-difference bound merges these two effects into a single probability mass. We therefore use it as a stability statement, and then refine it through the Q-Q orthogonality decomposition, which separates the total variation into the components $D_1$, $D_2$, and $D_3$. Thus, Q-Q orthogonality is not an alternative to the half-space framework, but its refinement: the former establishes stability, while the latter resolves its sources.

The bound in $L_1(P)$ admits the following interpretation. For the half-spaces $A=\{r\in\mathbb R^p:-\hat w^\top r<t\}$ and $B=\{r\in\mathbb R^p:-w^\top r<q_\alpha\}$, the indicator difference satisfies $\|1_A-1_B\|_{L_1(P)}=P(A\triangle B)$. By equation~\eqref{AgBnd}, $\|1_A-1_B\|_{L_1(P)}\le C'\sqrt{\|\hat w-w\|\mathbb E\|R\|+|t-q_\alpha|}$. The inequality shows that the two half-spaces are close in the $L_1(P)$ metric whenever the estimated direction $\hat w$ is close to $w$ and the threshold $t$ is close to the population quantile $q_\alpha$. In words, the symmetric difference $A\triangle B=(A\setminus B)\cup(B\setminus A)$ carries a small amount of probability mass. Since $A$ and $B$ are sublevel sets of the loss variable induced by the projection $L(w)=-w^\top R$, the bound asserts that the probability mass near the quantile level $q_\alpha$ is stable under perturbations in both the projection direction and the threshold.

\paragraph{Varying weight with fixed quantile.}
If the threshold is fixed at the population value, that is, if $t=q_\alpha$, then the term $|t-q_\alpha|$ disappears and the bound reduces to $\|\mathbf 1_{\{-\hat w^\top r<q_\alpha\}}-\mathbf 1_{\{-w^\top r<q_\alpha\}}\|_{L_1(P)}\le C'\sqrt{\|\hat w-w\|\mathbb E\|R\|}$. Fixing $t=q_\alpha$ therefore isolates the effect of perturbing the weight vector while keeping the quantile level unchanged. The bound measures how much probability mass shifts when the loss variable is changed from $-w^\top R$ to $-\hat w^\top R$ at the same reference threshold. In this sense, $\|\hat w-w\|$ governs the sensitivity of the sublevel set, and hence that of the quantile, to changes in the projection direction alone.

\paragraph{Varying quantile with fixed weights.}
If the weight vector is fixed at $\hat w=w$, then the bound becomes $\|\mathbf 1_{\{-w^\top r<t\}}-\mathbf 1_{\{-w^\top r<q_\alpha\}}\|_{L_1(P)}\le C'\sqrt{|t-q_\alpha|}$. This term measures the effect of shifting the threshold while holding the projection direction fixed. In this case the two half-spaces are parallel, and their symmetric difference is simply the set $\{r\in\mathbb R^p:q_\alpha\le -w^\top r<t\}\cup\{r\in\mathbb R^p:t\le -w^\top r<q_\alpha\}$. Thus, $|t-q_\alpha|$ captures pure quantile displacement along a fixed projected direction.

\paragraph{Jointly varying weight vector and quantile.}
When both the weight vector and the threshold are varied simultaneously, the bound is $\|\mathbf 1_A-\mathbf 1_B\|_{L_1(P)}\le C'\sqrt{\|\hat w-w\|\mathbb E\|R\|+|t-q_\alpha|}$. Therefore, the bound is controlled by the square root of the combined effects. This is sufficient for continuity and stability, since it guarantees that small perturbations in either component produce small perturbations in the induced half-space probability. Hence, it is a convenient tool for showing that the estimated Value-at-Risk remains stable in a neighborhood of the population quantile.

For this reason, inequality~\eqref{AgBnd} should be viewed as a stability statement rather than a full decomposition of variability. It ensures that the symmetric difference is small, but it does not identify the separate contributions arising from weight estimation and quantile estimation. This limitation motivates the Q-Q orthogonality decomposition in Section~\ref{QQ}, where these effects are separated and analyzed one at a time.

\paragraph{On the constants.}
Let $L=-w^\top R$, and suppose that the distribution function of $L$ is locally Lipschitz near $q_\alpha$, so that $P(|L-q_\alpha|\le u)\le \gamma u$ for all sufficiently small $u$, for some $\gamma>0$. In the proof, the random perturbation radius is $\eta=\|w-\hat w\|\|R\|+|t-q_\alpha|$, so that $\mathbb E[\eta]\le \|w-\hat w\|\mathbb E\|R\|+|t-q_\alpha|$. The local Lipschitz bound together with Markov's inequality gives $P(A\triangle B)\le \gamma\epsilon+\mathbb E[\eta]/\epsilon$. Choosing $\epsilon=\sqrt{\mathbb E[\eta]}$ yields $P(A\triangle B)\le C\sqrt{\mathbb E[\eta]}$, where $C>0$ absorbs the numerical constants arising from the local Lipschitz and Markov bounds. Consequently, $P(A\triangle B)\le C\sqrt{\|w-\hat w\|\mathbb E\|R\|+|t-q_\alpha|}$. If desired, the factor $\mathbb E\|R\|$ may also be absorbed into the constant $C$, yielding the equivalent form $P(A\triangle B)\le C\sqrt{\|w-\hat w\|+|t-q_\alpha|}$. In practice, $\mathbb E\|R\|$ may be replaced by the plug-in estimate $n^{-1}\sum_{i=1}^{n}\|R_i\|$.
\section{Multivariate $t_\nu$ Distribution and Variable Projections}
\label{Mult_t}

Let $R\in\mathbb R^p$ have law $P$. Throughout, we assume $R$ follows a multivariate Student $t_\nu$ distribution with location $\mu\in\mathbb R^p$, scatter matrix $\Sigma$, and degrees of freedom $\nu>2$. In particular, we may represent $R=\mu+\Sigma^{1/2}Z/\sqrt{S/\nu}$, where $Z\sim N_p(0,I_p)$, $S\sim\chi^2_\nu$, and $Z$ and $S$ are independent. Consequently, every one-dimensional projection $a^\top R$ is univariate $t_\nu$, with finite first moment and bounded density. This is precisely the regime we exploit, as it allows projected-density control without invoking rough bounds involving $\|R\|$.

For a measurable function $f$, we adopt the notation $Pf:=\int f\,dP$ and $P_nf:=n^{-1}\sum_{i=1}^n f(R_i)$, where $R_1,\dots,R_n$ are i.i.d.\ copies of $R$.

For $w\in\mathbb R^p$ and $t\in\mathbb R$, define the moving half-space indicator $g_{w,t}(r):=\mathbf 1\{-w^\top r<t\}$, $r\in\mathbb R^p$. Fix a reference pair $(w_0,q_0)$ and write $g_0:=g_{w_0,q_0}$. The associated sets are $A_{w,t}:=\{r:-w^\top r<t\}$ and $B:=\{r:-w_0^\top r<q_0\}$. The object of interest is $P|g_{w,t}-g_0|$, which is the population-level distance between indicator functions.

\subsection{Indicator Differences and Symmetric Difference}

We begin with the sub-level sets $A_{w,t}$ and $B_{w_0,q_0}$, while $g_{w,t}$ and $g_0$ are their corresponding indicator functions. In the following, all calculations are carried out pointwise in $r\in\mathbb R^p$.

\begin{lemma}[Indicator discrepancy equals symmetric difference]
For measurable sets $A,B\subseteq\mathbb R^p$, one has $|\mathbf 1_A-\mathbf 1_B|=\mathbf 1_{A\triangle B}$, where $A\triangle B:=(A\setminus B)\cup(B\setminus A)=(A\cap B^c)\cup(A^c\cap B)$. Consequently, $P|\mathbf 1_A-\mathbf 1_B|=P(A\triangle B)$. In particular, if $A_{w,t}:=\{r\in\mathbb R^p:-w^\top r\le t\}$, $B_{w_0,q_0}:=\{r\in\mathbb R^p:-w_0^\top r\le q_0\}$, $g_{w,t}:=\mathbf 1_{A_{w,t}}$, and $g_0:=\mathbf 1_{B_{w_0,q_0}}$, then $P|g_{w,t}-g_0|=P(A_{w,t}\triangle B_{w_0,q_0})$.
\end{lemma}

\begin{proof}
Fix $r\in\mathbb R^p$. There are four possibilities. If $r\in A\cap B$, then both indicators equal $1$, so $|\mathbf 1_A(r)-\mathbf 1_B(r)|=0$, and $r\notin A\triangle B$. If $r\in A^c\cap B^c$, then both indicators equal $0$, so again $|\mathbf 1_A(r)-\mathbf 1_B(r)|=0$, and $r\notin A\triangle B$. If $r\in A\setminus B$, then $\mathbf 1_A(r)=1$ and $\mathbf 1_B(r)=0$, so $|\mathbf 1_A(r)-\mathbf 1_B(r)|=1$, and $r\in A\triangle B$. If $r\in B\setminus A$, then $\mathbf 1_A(r)=0$ and $\mathbf 1_B(r)=1$, so $|\mathbf 1_A(r)-\mathbf 1_B(r)|=1$, and again $r\in A\triangle B$. These cases cover all four possibilities. Hence $|\mathbf 1_A-\mathbf 1_B|=\mathbf 1_{A\triangle B}$ pointwise. Integrating both sides with respect to $P$ gives $P|\mathbf 1_A-\mathbf 1_B|=P(A\triangle B)$, which proves the result.
\end{proof}

We now use this identity to reduce the symmetric difference of two half-spaces to a scalar comparison near the reference quantile. Let $U:=-w_0^\top R$ and $V:=-(w-w_0)^\top R$. Then $-w^\top R=-w_0^\top R-(w-w_0)^\top R=U+V$. Thus the perturbed loss $-w^\top R$ is decomposed into the reference loss $U$ and the weight-perturbation term $V$.

Because $R\sim t_\nu(\mu,\Sigma)$, every linear projection of $R$ is univariate Student $t_\nu$. In particular, $U=-w_0^\top R\sim t_\nu(-w_0^\top\mu,w_0^\top\Sigma w_0)$, and $V=-(w-w_0)^\top R\sim t_\nu(-(w-w_0)^\top\mu,(w-w_0)^\top\Sigma(w-w_0))$. Therefore, for each fixed $w$ near $w_0$, the reference projection $U$ has a bounded continuous density $f_U$, and the perturbation projection $V$ has a finite first absolute moment whenever $\nu>2$. Since we assume $\nu>2$, both first and second moments of these projections exist whenever the corresponding scale is finite.

\subsection{Symmetric Difference near the Reference Quantile Hyperplane}

In this subsection, we show that the symmetric difference of the two half-spaces is controlled by a one-dimensional interval around the reference quantile hyperplane.

\begin{lemma}[Bounding the symmetric difference near $q_0$]
\label{lem:slab-inclusion}
Let $A_{w,t}:=\{r\in\mathbb R^p:-w^\top r\le t\}$ and $B_{w_0,q_0}:=\{r\in\mathbb R^p:-w_0^\top r\le q_0\}$. Define $U:=-w_0^\top R$ and $V:=-(w-w_0)^\top R$. Then, for every $w\in\mathbb R^p$ and $t\in\mathbb R$, $A_{w,t}\triangle B_{w_0,q_0}\subseteq\{|U-q_0|\le |V|+|t-q_0|\}$.
\end{lemma}

\begin{proof}
Take $r\in A_{w,t}\triangle B_{w_0,q_0}$. Then $r$ lies in exactly one of the two sets $A_{w,t}$ and $B_{w_0,q_0}$. Hence there are two cases.

\noindent\textbf{Case 1: $r\in A_{w,t}\setminus B_{w_0,q_0}$.} Then $-w^\top r\le t$ and $-w_0^\top r>q_0$. Since $U(r)=-w_0^\top r$, this gives $U(r)>q_0$, and hence $0<U(r)-q_0$. Add and subtract $-w^\top r$ and $t$ to obtain $U(r)-q_0=[(-w_0^\top r)-(-w^\top r)]+[(-w^\top r)-t]+(t-q_0)$. Since $(-w^\top r)-t\le0$, we may discard this nonpositive term and obtain $U(r)-q_0\le |(-w_0^\top r)-(-w^\top r)|+|t-q_0|$. Also, $|(-w_0^\top r)-(-w^\top r)|=|(w-w_0)^\top r|=|V(r)|$. Therefore $0<U(r)-q_0\le |V(r)|+|t-q_0|$, and hence $|U(r)-q_0|\le |V(r)|+|t-q_0|$.

\noindent\textbf{Case 2: $r\in B_{w_0,q_0}\setminus A_{w,t}$.} Then $-w_0^\top r\le q_0$ and $-w^\top r>t$. Since $U(r)=-w_0^\top r$, this gives $U(r)\le q_0$, and hence $0\le q_0-U(r)$. Add and subtract $-w^\top r$ and $t$ to obtain $q_0-U(r)=(q_0-t)+[t-(-w^\top r)]+[(-w^\top r)-(-w_0^\top r)]$. Since $t-(-w^\top r)<0$, we may discard this negative term and obtain $q_0-U(r)\le |q_0-t|+|(-w^\top r)-(-w_0^\top r)|$. Also, $|(-w^\top r)-(-w_0^\top r)|=|(w_0-w)^\top r|=|V(r)|$. Therefore $0\le q_0-U(r)\le |t-q_0|+|V(r)|$, and hence $|U(r)-q_0|\le |V(r)|+|t-q_0|$.

Since every $r\in A_{w,t}\triangle B_{w_0,q_0}$ falls into one of these two cases, we have shown that $r\in A_{w,t}\triangle B_{w_0,q_0}$ implies $|U(r)-q_0|\le |V(r)|+|t-q_0|$. This proves the inclusion.
\end{proof}

\subsection{Multivariate $t$-distribution and the Population Bound}

\begin{theorem}[Population symmetric difference bound under $t_\nu$]
\label{thm:population-wedge-fubini}
Assume $R\sim t_\nu(\mu,\Sigma)$ with $\nu>2$. Fix $w_0\neq0$, let $q_0$ be the $\alpha$-quantile of $U=-w_0^\top R$, and define $g_{w,t}(R)=\mathbf 1\{-w^\top R\le t\}$ and $g_0(R)=\mathbf 1\{-w_0^\top R\le q_0\}$. For $w$ near $w_0$ and $t$ near $q_0$, set $V=-(w-w_0)^\top R$. Suppose that there exist constants $\delta_0>0$ and $M<\infty$ such that, for all $0\le a\le\delta_0$ and all relevant values of $v$, $P(|U-q_0|\le a\mid V=v)\le2Ma$. Equivalently, whenever the conditional density exists, it is locally bounded near $q_0$ uniformly in $v$; in the degenerate case $w=w_0$, the condition is interpreted through the marginal density of $U$ near $q_0$. Then, for some finite constant $C>0$,
\begin{equation}
\label{PB_t_Dist}
P|g_{w,t}-g_0|
\le
C\left\{|t-q_0|+|(w-w_0)^\top\mu|+\left((w-w_0)^\top\Sigma(w-w_0)\right)^{1/2}\right\}.
\end{equation}
In particular, for fixed local $\mu$ and $\Sigma$, $P|g_{w,t}-g_0|\le C\{|t-q_0|+\|w-w_0\|\}$.
\end{theorem}

\begin{proof}
Set $U=-w_0^\top R$ and $V=-(w-w_0)^\top R$, so that $-w^\top R=U+V$. Then $g_0(R)=\mathbf 1\{U\le q_0\}$ and $g_{w,t}(R)=\mathbf 1\{U+V\le t\}$. By Lemma 3 if these indicators disagree, then either $U\le q_0<U+V+(q_0-t)$ or $U>q_0\ge U+V+(q_0-t)$, and in either case $|U-q_0|\le |t-q_0|+|V|$. Hence $P|g_{w,t}-g_0|\le P\{|U-q_0|\le |t-q_0|+|V|\}$.

By the law of total probability with respect to $V$, followed by the conditional slab bound, $P\{|U-q_0|\le |t-q_0|+|V|\}=\int P\{|U-q_0|\le |t-q_0|+|v|\mid V=v\}\,dP_V(v)\le\int2M\{|t-q_0|+|v|\}\,dP_V(v)=2M|t-q_0|+2M E|V|$. Therefore $P|g_{w,t}-g_0|\le2M|t-q_0|+2M E|(w-w_0)^\top R|$.

Since $R\sim t_\nu(\mu,\Sigma)$, the projection $(w-w_0)^\top R$ is univariate $t_\nu$-type with location $(w-w_0)^\top\mu$ and scale $\{(w-w_0)^\top\Sigma(w-w_0)\}^{1/2}$. Because $\nu>2$, hence $\nu>1$, the first absolute moment is finite; hence, for some $C_\nu<\infty$, $E|(w-w_0)^\top R|\le C_\nu\{|(w-w_0)^\top\mu|+((w-w_0)^\top\Sigma(w-w_0))^{1/2}\}$. Combining the preceding inequalities and absorbing constants proves the first claim. The local bound follows from $|(w-w_0)^\top\mu|\le\|\mu\|\|w-w_0\|$ and $\{(w-w_0)^\top\Sigma(w-w_0)\}^{1/2}\le\|\Sigma\|_{\mathrm{op}}^{1/2}\|w-w_0\|$.
\end{proof}

\subsection{Comparison of the generic and multivariate $t$ symmetric-difference bounds}

After deriving the bounds, it is useful to compare the generic symmetric-difference bound in Lemma~\ref{HalfSpaceLipschitz}, equation~\eqref{AgBnd}, with its multivariate $t$ specialization in Theorem~\ref{thm:population-wedge-fubini}, equation~\eqref{PB_t_Dist}. While both control $P(A\triangle B)$ in terms of perturbations in the weight vector and threshold, they differ in structure, sharpness, and the extent to which the heavy-tailed geometry is made explicit.

The \emph{generic} bound is distributionally broad. It requires only a finite first moment, $\mathbb E\|R\|<\infty$, together with a local Lipschitz property of the projected distribution near the reference quantile $q_0$. Its strength lies in its robustness: it applies beyond any specific parametric model. However, this generality comes at a cost. The constant $C$ is abstract and absorbs several features of the underlying law into a single unspecified quantity. In particular, the proof proceeds through the estimate $\mathbb E|(w-w_0)^\top R|\le\|w-w_0\|\,\mathbb E\|R\|$, so that the resulting bound controls the symmetric-difference slab only indirectly through a moment condition. From the standpoint of portfolio risk under heavy tails, this is adequate for continuity but somewhat unsatisfactory for interpretation.

By contrast, the \emph{multivariate $t$} specialization retains the same first-order structure while making the tail geometry of the problem substantially more transparent. Because every linear projection of a multivariate $t$ random vector is again univariate Student $t$, the projected loss $L_0=-w_0^\top R$ has an explicit density. As a result, the probability mass of the symmetric-difference slab can be bounded directly through the projected density $f_{L_0}$, therefore through the quantity $M(w_0)=\sup_{x\in\mathbb R}f_{L_0}(x)$. This replaces the opaque local regularity constant in the generic argument by an explicit model-based quantity depending on the projected scale $w_0^\top\Sigma w_0$ and the degrees of freedom $\nu$. Moreover, under $\nu>2$, the finiteness of $\mathbb E\|R\|$ is automatic, so the moment condition required by the generic bound is validated inside the heavy-tailed model itself.

Thus, the generic bound and the multivariate $t$ bound serve complementary purposes. The former establishes broad continuity of the half-space class under minimal assumptions; the latter shows how, in a standard heavy-tailed model, this continuity can be tied directly to the projected loss distribution. The generic argument is more flexible, whereas the multivariate $t$ argument is more explicit and better aligned with the tail-risk interpretation of the problem.

At the same time, both bounds share a common limitation. Each controls only the \emph{aggregate} symmetric-difference probability $P(A\triangle B)$, and therefore merges into one quantity two conceptually distinct perturbations: movement in the weight vector $w$ and movement in the threshold $t$. In other words, both bounds certify stability, but neither disentangles the separate contributions arising from estimated weights and estimated quantiles. This is precisely the point at which the half-space approach reaches its natural limit. To understand the variability of $\widehat{\mathrm{VaR}}_\alpha$ more finely, a decomposition is needed that isolates these contributions rather than masking them inside a single slab probability. That is the role of the Q-Q orthogonality formulation introduced next.
\section{Glivenko--Cantelli, Uniform Convergence, and Empirical Quantile Bounds}
\label{GC}

The symmetric-difference bounds established in the preceding section are population-level bounds. They control the quantity $P(A\triangle B)$, where $A$ and $B$ are half-spaces indexed by perturbations in the weight vector and threshold. While this is necessary for stability, it is not sufficient for the analysis of empirical quantiles and the estimated Value-at-Risk. Since the empirical quantities are built from the sample, we must also establish that empirical probabilities track their population analogues uniformly over the relevant class of sets. To do so requires ergodicity together with Glivenko--Cantelli convergence. The resulting framework follows the classical empirical-process viewpoint developed in  \cite{pollard1984,dudley1999,vandervaartwellner2023,vandervaart1998}.

\subsection{Binomial distribution and symmetric difference}

Let $A=\{r\in\mathbb R^p:-w^\top r<t\}$ and $B=\{r\in\mathbb R^p:-w_0^\top r<q_0\}$ be two half-spaces. For each observation $R_i$, the random variable $\mathbf 1_{A\triangle B}(R_i)$ is Bernoulli with success probability $P(A\triangle B)$, and the empirical symmetric difference is given by
$P_n(A\triangle B)=n^{-1}\sum_{i=1}^n\mathbf 1_{A\triangle B}(R_i)$,
that is, the empirical proportion of observations falling in $A\triangle B$.

More precisely, any pair $A,B\in\mathcal H$ partitions the sample space into four disjoint subsets given by
$S_{11}=A\cap B$,
$S_{10}=A\cap B^c$,
$S_{01}=A^c\cap B$,
and
$S_{00}=A^c\cap B^c$.
The corresponding empirical counts
$N_{jk}=\sum_{i=1}^n\mathbf 1_{\{R_i\in S_{jk}\}}$,
for $(j,k)\in\{0,1\}^2$, induce the empirical measure $P_n$, and in particular
$P_n(A\triangle B)=(N_{10}+N_{01})/n$.
Thus, for each fixed pair $(w,t)$, the symmetric difference is the fraction of observations that fall in $A\triangle B$, and viewing the half-space indicators as a family of Bernoulli random variables indexed by $(w,t)$ provides a useful empirical-process interpretation, wherein the continuum of indicator proportions is tracked as $(w,t)$ varies.

Let
$\mathcal H=\{H_{w,t}:H_{w,t}=\{r\in\mathbb R^p:-w^\top r<t\},\;w\in\mathbb R^p,\;t\in\mathbb R\}$
denote the class of half-spaces, and let
$\mathcal F=\{\mathbf 1_{H_{w,t}}:H_{w,t}\in\mathcal H\}$
denote the associated class of indicator functions induced by the projected loss variable $L=-w^\top R$.

For each $(w,t)$, define the population CDF
$F(w,t):=P(-w^\top R<t)$
and its empirical analogue
$F_n(w,t):=n^{-1}\sum_{i=1}^n\mathbf 1\{-w^\top R_i<t\}$.

\subsection{Uniform convergence and VC--Glivenko--Cantelli theory}

The objective is to evaluate the stability of half-spaces generated by pairs $(w,t)$ in a neighborhood of a reference pair $(w_0,q_0)$. The empirical quantile is defined through the relation
$F_n(t;w)=\alpha$,
where
$F_n(t;w)=n^{-1}\sum_{i=1}^n\mathbf 1\{-w^\top R_i\le t\}$.

For each fixed pair $(w,t)$, the variables $\mathbf 1\{-w^\top R_i\le t\}$ are i.i.d. Bernoulli with mean
$F(t;w)=P(-w^\top R\le t)$,
and by the Strong Law of Large Numbers,
$F_n(t;w)\to F(t;w)$
almost surely for each fixed $(w,t)$. This gives pointwise convergence. However, because $(w,t)$ varies over a continuum, pointwise convergence alone does not imply convergence of
$\sup_{(w,t)}|F_n(t;w)-F(t;w)|$.

To obtain simultaneous control over neighborhoods of $(w_0,q_0)$, we invoke uniform convergence over the class
$\mathcal F=\{\mathbf 1\{-w^\top r\le t\}:w\in\mathbb R^p,\;t\in\mathbb R\}$,
namely,
\begin{equation}
\sup_{(w,t)}|F_n(t;w)-F(t;w)|\to0
\qquad\text{almost surely}.
\label{UnifConv}
\end{equation}

The class $\mathcal F$ is the class of half-space indicators in $\mathbb R^p$, hence it is a VC class. By the Glivenko--Cantelli theorem for VC classes, as developed in empirical-process theory \cite{vapnik1971uniform,pollard1984,vandervaartwellner2023,dudley1999,vandervaart1998}, the required uniform convergence follows. This uniform control permits the stability analysis of the data-dependent pair $(\hat w,\hat q)$ in a neighborhood of $(w_0,q_0)$.

Since every member of $\mathcal F$ is an indicator function,
$|\mathbf 1_{H_{w,t}}(r)|\le1$
for all $r\in\mathbb R^p$ and all $(w,t)$. Thus, the constant function
$F(r)\equiv1$
serves as an envelope for the class, and in particular
$\sup_{f\in\mathcal F}|f(r)|\le1$
for all $r\in\mathbb R^p$. Therefore,
$\mathcal F\subset L_1(P)$,
and the envelope is integrable since
$\int1\,dP=1$.

Recall that Value-at-Risk at level $\alpha\in\{0.90,0.95,0.99\}$ is a tail-risk measure based on returns $R_1,\dots,R_n$, which in financial applications are often modeled not as independent observations but as a weakly dependent time series. It is standard to assume that the return sequence is weakly stationary and ergodic with common probability law $P$, so that time averages converge to their population expectations. In particular, for each fixed half-space $H$,
$n^{-1}\sum_{i=1}^n\mathbf 1\{R_i\in H\}\xrightarrow{\mathrm{a.s.}}P(H)$.
Thus, ergodicity ensures the law of large numbers for each fixed $H$. The discussion of weak dependence and ergodicity is intended primarily as motivation arising from financial return series, where exact independence is often unrealistic in practice.  For uniform control over a continuum of half-spaces, however, we invoke Glivenko--Cantelli-type results under standard assumptions, such as i.i.d.\ observations, for the class $\mathcal F$. In the spirit of full disclosure, extending empirical-process arguments to dependent sequences is beyond the scope of the current work. Here, the role of Glivenko--Cantelli theory is primarily to provide the classical empirical-process scaffolding needed to transfer population-level stability bounds to their empirical analogues.  For the formal empirical-process arguments developed here, we work under the classical i.i.d. framework associated with VC--Glivenko--Cantelli theory. The discussion of weak dependence and ergodicity is intended primarily as motivation arising from financial return series, where exact independence is often unrealistic in practice. Extending the present framework to dependent empirical processes under mixing or related conditions is an important direction for future work.

\subsection{Application to symmetric difference $A\triangle B$}

For $A=\{r:-w^\top r<t\}$ and $B=\{r:-w_0^\top r<q_0\}$, we have
$P(A\triangle B)=\|1_A-1_B\|_{L_1(P)}$ and
$P_n(A\triangle B)=\|1_A-1_B\|_{L_1(P_n)}$; see Lemmas~5 and~6 in Appendix~A.

Therefore, the symmetric-difference probability is the mean of a Bernoulli indicator, and its empirical analogue is the corresponding binomial proportion.  Once a population bound of the form
\(P(A\Delta B)\le C(\|w-w_0\|+|t-q_0|)\) has been established, as in Lemma~1,
the Glivenko--Cantelli property ensures that
\(P_n(A\Delta B)=P(A\Delta B)+o_P(1)\) uniformly over the relevant neighborhood
of \((w_0,q_0)\). In this sense, each symmetric difference defines a binomial experiment, while uniform convergence provides the control required over the full continuum of perturbations.

\subsection{Empirical stability of symmetric differences}

Uniform convergence transfers the population symmetric-difference bound in Lemma~1 to the empirical setting for quantile estimation. Specifically, for sets
$A=\{r:-w^\top r<t\}$
and
$B=\{r:-w_0^\top r<q_0\}$,
the population argument shows that small perturbations in $(w,t)$ yield
$P(A\triangle B)\le C(\|w-w_0\|+|t-q_0|)$.
By the Glivenko--Cantelli property of the class $\mathcal F$, we have
$\sup_{(w,t)}|P_n(A\triangle B)-P(A\triangle B)|\to0$,
so that
$P_n(A\triangle B)\le C(\|w-w_0\|+|t-q_0|)+o_P(1)$
uniformly over the relevant neighborhood of $(w_0,q_0)$.

Thus, ``small'' perturbation is quantified in terms of the perturbation magnitude
$\|w-w_0\|+|t-q_0|$,
up to an asymptotically negligible term. In this sense, the geometric stability of the population distribution is inherited by the empirical distribution function.

Uniform convergence, bounds on the class $\mathcal F$, and the invocation of the Glivenko--Cantelli theorem therefore serve as the empirical-process scaffolding underlying stability of the sample quantile. The symmetric-difference lemma quantifies the geometry of perturbations, while ergodicity and Glivenko--Cantelli convergence justify replacing the population measure by its empirical analogue. Together they provide the framework for the subsequent decomposition of variation.

\subsection{Limitations of the half-space bound}

Even with uniform convergence in hand, the half-space bound yields an aggregate bound on the total discrepancy between the perturbed and reference half-spaces, but does not distinguish variability arising from changes in the weight vector from that due to changes in the threshold. While uniform convergence guarantees stability, it does not separate the sources of variation attributable to $w$ and $t$. For this reason, we seek a decomposition that isolates these distinct contributions, leading to the Q-Q orthogonality formulation introduced next.
\section{All Asymptotics in Local:  Variability Decomposition and Q-Q Orthogonality} 
\label{QQ}

The half-space and symmetric-difference bounds developed earlier provide control of the aggregate discrepancy between perturbed and reference loss regions and are sufficient for establishing continuity and stability. However, they do not resolve the sources of variation that arise when the loss functional depends jointly on an estimated weight vector and an estimated quantile. In particular, a bound of the form $P(A\triangle B)\le C(\|\hat w-w\|+|t-q_\alpha|)$ proves that the symmetric-difference region is small, but does not indicate whether the discrepancy is driven by variation in the weight vector, variation in the quantile, or residual empirical fluctuation inherent to the Bahadur approximation.

We address this limitation through what we term Q-Q orthogonality, which yields a decomposition that separates these contributions in the spirit of ANOVA.
The term orthogonality is used here in a coordinate-wise local sense. The construction varies the projection direction and the quantile threshold one at a time relative to the reference pair $(w_0,q_0)$. Under the tangent-space restrictions introduced below, the first-order directional contribution in the weight coordinate vanishes, leaving the quantile coordinate to carry the leading local movement. Thus, the terminology refers not to global Hilbert-space orthogonality, but to a local separation of first-order effects. Classical ANOVA relies on $L_2(P)$ orthogonality, whereas our analysis is conducted in $L_1(P)$, where $\|1_A-1_B\|_{L_1(P)}=P(A\triangle B)$ bounds the difference. Variability is decomposed additively in terms of perturbations in the projection direction $w$ and the threshold $t$, with the directional derivative in the weight coordinate vanishing at first order under the tangent-space conditions. This distinction is particularly relevant in heavy-tailed settings, where second moments may be unstable or may not exist.

The terminology reflects the fact that the construction aligns population and empirical quantiles relative to a fixed reference pair by varying one argument at a time. In contrast to a graphical Q-Q plot, the present formulation is analytical and is designed to isolate the respective contributions of the quantile coordinate and the projection direction.

\subsection{First-order orthogonality under tangent-space perturbations}

\noindent\textbf{Proposition (First-order orthogonality under tangent-space perturbations).}
Let $F(w,q):=P(-w^\top R\le q)$, with $R\sim t_\nu(\mu,\Sigma)$ for $\nu>2$, and suppose $w_0^\top\Sigma w_0>0$. Let $(w_0,q_0)$ satisfy $F(w_0,q_0)=\alpha$. Writing $\partial_wF(w_0,q_0)h:=\left.dF(w_0+\varepsilon h,q_0)/d\varepsilon\right|_{\varepsilon=0}$ for the Gateaux derivative in the weight argument, and writing $\partial_qF(w_0,q_0)$ for the ordinary derivative in the threshold argument, consider perturbations $w_\varepsilon=w_0+\varepsilon h$ and $q_\delta=q_0+\delta$, with $h\in\mathbb R^p$. Then
\begin{equation}
F(w_0+\varepsilon h,q_0+\delta)-F(w_0,q_0)
=
\varepsilon\,\partial_wF(w_0,q_0)h
+
\delta\,\partial_qF(w_0,q_0)
+
o(|\varepsilon|+|\delta|).
\label{FirstOrderExpansion}
\end{equation}
The tangent-space conditions $h^\top\mu = 0$ and $h^\top\Sigma w_0 = 0$ are introduced to characterize directions along which the first-order perturbation of the projected loss distribution vanishes. The proposition is therefore geometric and local in nature. We do not assume that a generic estimator $\hat w$ necessarily fluctuates exactly within this tangent space. Rather, the result identifies directions for which the first-order contribution of the weight perturbation disappears, thereby isolating the quantile coordinate in the local expansion.
\begin{proof}
Under the multivariate $t$ model, the projected loss $L(w)=-w^\top R$ is univariate $t_\nu$ with projected location $-w^\top\mu$ and projected scale $(w^\top\Sigma w)^{1/2}$. Therefore, $F(w,q)=T_\nu\{(q+w^\top\mu)/(w^\top\Sigma w)^{1/2}\}$, where $T_\nu$ is the standardized univariate $t_\nu$ distribution function. Let $z(w,q):=(q+w^\top\mu)/(w^\top\Sigma w)^{1/2}$, so that $F(w,q)=T_\nu\{z(w,q)\}$.

Differentiating along the path $w_\varepsilon=w_0+\varepsilon h$, with $q=q_0$ fixed, gives $z(w_\varepsilon,q_0)=\{q_0+(w_0+\varepsilon h)^\top\mu\}/\{(w_0+\varepsilon h)^\top\Sigma(w_0+\varepsilon h)\}^{1/2}$. Writing $z_0=z(w_0,q_0)$ and $s_0=(w_0^\top\Sigma w_0)^{1/2}$, direct differentiation at $\varepsilon=0$ yields
\begin{equation}
\left.\frac{d}{d\varepsilon}z(w_\varepsilon,q_0)\right|_{\varepsilon=0}
=
\frac{h^\top\mu}{s_0}
-
\frac{(q_0+w_0^\top\mu)(h^\top\Sigma w_0)}{s_0^3}.
\label{zdiff}
\end{equation}
Hence, by the chain rule,
\begin{equation}
\partial_wF(w_0,q_0)h
=
t_\nu(z_0)
\left\{
\frac{h^\top\mu}{s_0}
-
\frac{(q_0+w_0^\top\mu)(h^\top\Sigma w_0)}{s_0^3}
\right\},
\label{DirectionalDerivative}
\end{equation}
where $t_\nu$ is the standardized univariate $t_\nu$ density. Similarly, $\partial_qF(w_0,q_0)=t_\nu(z_0)/s_0$. If $h^\top\mu=0$ and $h^\top\Sigma w_0=0$, then $\partial_wF(w_0,q_0)h=0$. Consequently, the expansion reduces to
\begin{equation}
F(w_0+\varepsilon h,q_0+\delta)-F(w_0,q_0)
=
\delta\,\partial_qF(w_0,q_0)
+
o(|\varepsilon|+|\delta|).
\label{OrthogonalExpansion}
\end{equation}
Clearly, along tangent directions satisfying $h^\top\mu=0$ and $h^\top\Sigma w_0=0$, the first-order movement of $F(w,q)$ is captured entirely in the quantile coordinate; the directional perturbation in $w$ is reflected through the $o(|\varepsilon|+|\delta|)$ term.
\end{proof}

\subsubsection*{A warning note on tangent-space restrictions}

While the population vector $\mu$ and matrix $\Sigma$ are fixed, changes in the weight vector alter the projected location $-w^\top\mu$ and projected scale $(w^\top\Sigma w)^{1/2}$ of the loss $L(w)=-w^\top R$. Along $w_\varepsilon=w_0+\varepsilon h$, the first-order change in projected location is $-h^\top\mu$, while the first-order change in $w^\top\Sigma w$ is $2h^\top\Sigma w_0$. The restriction $h^\top\mu=0$ removes the first-order location movement, while $h^\top\Sigma w_0=0$ removes the first-order scale movement. Equivalently, these restrictions define the tangent space $\mathcal T_{w_0}:=\{h\in\mathbb R^p:h^\top\mu=0,\ h^\top\Sigma w_0=0\}$.  For $h\in\mathcal T_{w_0}$, the projected univariate $t$ distribution of the loss $-w^\top R$ is unchanged to first order, so the first-order movement of $F(w,q)$ is extracted in the quantile direction. In this sense, the weight and quantile movements separate locally.  See \cite{vandervaart1998} and \cite{bickel1993} for the local perturbation and tangent-space viewpoint in asymptotic statistics.
In practice, estimated portfolio weights may not satisfy these restrictions exactly; however, the tangent-space formulation isolates directions for which projection-induced variation is orthogonal to first-order quantile variation.
\subsection{Population and empirical quantiles indexed by the weight vector}

For each nonzero weight vector $w\in\mathbb R^p$, let  $L(w):=-w^\top R$ denote projection-induced loss and let its population distribution function be $F(w,t):=P(L(w)\le t)=P(-w^\top R\le t)$. At the fixed level $\alpha\in(0,1)$ the population $\alpha$-quantile associated with $w$ is defined by $q_\alpha(w):=\inf\{t\in\mathbb R:F(w,t)\ge\alpha\}$. Thus,  $ q_\alpha(w)$ is the value at risk (VaR) level for each projection direction.  Let $R_1,\dots,R_n$ denote the observed returns and define the empirical distribution function associated with the weight vector $w$ by $F_n(w,t):=n^{-1}\sum_{i=1}^n\mathbf 1\{-w^\top R_i\le t\}$. The empirical $\alpha$-quantile indexed by $w$ is then $\hat q_\alpha(w):=\inf\{t\in\mathbb R:F_n(w,t)\ge\alpha\}$.  In particular, if $w_0$ denotes a reference population weight vector and $\hat w$ is its estimator, then there are three quantiles in play: $q_\alpha(w_0)$, $q_\alpha(\hat w)$, and $\hat q_\alpha(\hat w)$. The first is the reference population quantile, the second is the population quantile under the perturbed direction $\hat w$, and the third is the empirical quantile based on both the perturbed direction and the sample.  This distinction is important. If we compare only $\hat q_\alpha(\hat w)$ with $q_\alpha(w_0)$, then all sources of variation are encapsulated in a single quantity. The Q-Q orthogonality decomposition inserts the intermediate population term $q_\alpha(\hat w)$ to separate the perturbation induced by the weight estimate from that induced by empirical quantile estimation.

\subsection{Reference pair and one-coordinate-at-a-time decomposition}

Fix a reference pair $(w_0,q_0)$, where $q_0:=q_\alpha(w_0)$ is the population $\alpha$-quantile of the reference loss $L_0:=-w_0^\top R$. The  empirical object of interest is $\hat q_\alpha(\hat w)$, defined by $F_n(\hat w,\hat q_\alpha(\hat w))\ge\alpha$, and, under continuity, given locally by the approximate equality $F_n(\hat w,\hat q_\alpha(\hat w))=\alpha$.  We compare $\hat q_\alpha(\hat w)$ with the reference quantile $q_0$ by adding and subtracting the intermediate population quantile $q_\alpha(\hat w)$. This gives the exact algebraic identity
\begin{equation}
\hat q_\alpha(\hat w)-q_0
=
\bigl[q_\alpha(\hat w)-q_0\bigr]
+
\bigl[\hat q_\alpha(\hat w)-q_\alpha(\hat w)\bigr].
\label{Qhat_what}
\end{equation}
Equation \eqref{Qhat_what} separates the total deviation into two pieces: the population quantile movement caused by changing the direction from $w_0$ to $\hat w$, and the empirical quantile estimation error after the direction has been fixed at $\hat w$.

To linearize the second term in \eqref{Qhat_what}, we apply Bahadur's representation \cite{bahadur1966} in the perturbed direction $\hat w$: 
\begin{equation}
\hat q_\alpha(\hat w)-q_\alpha(\hat w)
=
\frac{\alpha-F_n(\hat w,q_\alpha(\hat w))}
{f_{\hat w}(q_\alpha(\hat w))}
+
R_n,
\label{RRB}
\end{equation}
where $f_{\hat w}(q_\alpha(\hat w))$ denotes the density of the projected loss $-\hat w^\top R$, evaluated at its population quantile $q_\alpha(\hat w)$, and $R_n$ is the Bahadur remainder term in the perturbed direction $\hat w$.  Since the projection direction is itself estimated, the Bahadur expansion must hold uniformly over directions $w$ in a neighborhood of $w_0$, rather than merely pointwise for a fixed distribution. Because the indexed class of half-space indicators
$\mathcal{F}=\{1\{-w^\top r\le t\}: w\in\mathbb{R}^p,t\in\mathbb{R}\}$
is a VC class, standard empirical-process results yield uniform control of
the empirical process over local neighborhoods of $(w_0,q_0)$; see
Pollard (1984), van der Vaart and Wellner (1996), and van der Vaart (1998).
Combined with
classical quantile-process arguments and Bahadur-type expansions
(Bahadur, 1966; Kiefer, 1967), this uniform empirical-process control yields
a local uniform representation for the empirical quantile process. 
 Consequently, the remainder term $R_n(w)$ is uniform over $w$ in a neighborhood of $w_0$; see  \cite{kiefer1967}. This justifies evaluating the Bahadur expansion at the random perturbed direction $\hat w$.  Equivalently, the decomposition may be written as
\begin{equation}
\hat q_\alpha(\hat w)-q_\alpha(\hat w)
=
\frac{\alpha-F_n(\hat w,q_\alpha(\hat w))}
{f_{\hat w}(q_\alpha(\hat w))}
+
R_n(\hat w),
\end{equation}
where the remainder satisfies $R_n(\hat w)=o_p(n^{-1/2})$ uniformly over local perturbations of $w_0$.  Substituting (10) into (8) yields the Q-Q orthogonality decomposition
\begin{equation}
\hat q_\alpha(\hat w)-q_0
=
\underbrace{q_\alpha(\hat w)-q_0}_{D_1}
+
\underbrace{
\frac{\alpha-F_n(\hat w,q_\alpha(\hat w))}
{f_{\hat w}(q_\alpha(\hat w))}
}_{D_2}
+
\underbrace{R_n(\hat w)}_{D_3}.
\tag{11}
\end{equation} 
The decomposition is algebraically simple: it inserts the intermediate population quantile $q_\alpha(\hat w)$ between the reference quantile $q_\alpha(w_0)$ and the empirical quantile $\hat q_\alpha(\hat w)$. Its usefulness lies not in the algebra itself, but in the separation it induces. The term $D_1$ isolates the population-level movement caused by perturbing the projection direction, while $D_2$ and $D_3$ retain the empirical fluctuation and Bahadur remainder once the direction is fixed at $\hat w$. We therefore use the term Q-Q orthogonality in a local coordinate sense, rather than as an assertion of global $L^2$ orthogonality or independence among the components.  In the following we discuss the meaning of each of the components $\{D_i\}_{i=1}^{3}$ in detail.

\subsubsection*{Component $D_1$: perturbation in the weight vector with the quantile held at population level}

 The term $D_1=q_\alpha(\hat w)-q_\alpha(w_0)$ captures the change in the population quantile induced solely by replacing the reference weight vector $w_0$ by its estimator $\hat w$. No empirical fluctuation is present in this term. Both quantities are population level quantiles. So $D_1$ measures the sensitivity of the quantile $q_\alpha(w)$ to perturbation in the projection direction $w$. This is precisely the contribution that is obscured inside the generic symmetric-difference bound. In the half-space formulation, changing $w$ rotates or tilts the boundary $\{-w^\top r=t\}$ and therefore changes the induced loss distribution. However, the symmetric-difference probability sees only the resulting aggregate symmetric difference and does not tease out how much of the quantile movement is due to the directional perturbation itself. The term $D_1$ makes that contribution explicit. 
 Geometrically, $D_1$ describes what happens when the loss axis is changed from $r\mapsto -w_0^\top r$ to $r\mapsto -\hat w^\top r$, while remaining entirely at the population level. In this sense, it is the directional or weight-induced component of variability. If the map $w\mapsto q_\alpha(w)$ is differentiable at $w_0$, then $D_1$ admits the first-order expansion $D_1=\nabla_w q_\alpha(w_0)^\top(\hat w-w_0)+o(\|\hat w-w_0\|)$. Even when such differentiability is not invoked explicitly, the earlier slab bounds imply continuity of this component. That is, the quantile changes only slightly when the direction changes slightly, provided the projected density remains positive and regular near the target quantile.

\subsubsection*{Component $D_2$: empirical quantile fluctuation with the direction fixed}

The second term $D_2=\{\alpha-F_n(\hat w,q_\alpha(\hat w))\}/f_{\hat w}(q_\alpha(\hat w))$ captures the empirical fluctuation of the quantile once the direction has already been fixed at $\hat w$. That is; in this term the weight vector is held fixed and only the empirical quantile varies around the true quantile. This is the part of the decomposition that is closest in spirit to classical quantile asymptotics. Indeed, for a fixed $w$, Bahadur's representation asserts that $\hat q_\alpha(w)-q_\alpha(w)=\{\alpha-F_n(w,q_\alpha(w))\}/f_w(q_\alpha(w))+R_n(w)$, so $D_2$ is simply the leading stochastic linear term evaluated at the perturbed direction $\hat w$. The numerator $\alpha-F_n(\hat w,q_\alpha(\hat w))$ measures the difference between the empirical cdf and its target level $\alpha$ at the true quantile corresponding to the direction $\hat w$. The denominator $f_{\hat w}(q_\alpha(\hat w))$ rescales that difference by the local slope of the projected distribution function $F_{\hat w}(q_\alpha(\hat w))$.

Thus $D_2$ is the quantile-estimation component of variability. It reflects
sampling fluctuation in the empirical cdf and the resulting variability of the
sample quantile, but unlike $D_1$, it does not describe changes in the loss
distribution induced by perturbations in the weight vector. In this component,
the projection direction is treated as fixed. This distinction is precisely
what makes the decomposition useful: without separating $D_1$ and $D_2$, one
cannot distinguish variability arising from instability in the estimated
weights from the usual empirical fluctuation of sample quantiles.

\subsubsection*{Component $D_3$: Bahadur remainder}

The final term $D_3=R_n$ is the residual term in the Bahadur representation. It collects the higher-order part of the quantile expansion that is not captured by the linear empirical-process term $D_2$. In the classical fixed-direction setting, one typically writes $R_n=O_p(n^{-3/4}\log n)$ under appropriate regularity conditions. In the present setting, the direction is itself random, so the remainder must be understood as the residual error after linearization at the perturbed direction $\hat w$. Its role in the decomposition is important conceptually even when it is asymptotically smaller than $D_1$ and $D_2$. It records the fact that the empirical quantile is not exactly linear in the empirical cdf and therefore that any first-order decomposition necessarily leaves a remainder. The term $D_3$ should not be ignored simply because it is asymptotically smaller. As we are concerned with stability, it is essential to call out the remainder explicitly. Otherwise the decomposition would incorrectly suggest that the quantile fluctuation is encapsulated by the leading empirical-process term.

\subsection{Why this is an orthogonality decomposition}

The terminology Q-Q orthogonality refers to the fact that the decomposition moves one coordinate at a time relative to the reference pair $(w_0,q_0)$. First, the projection direction is changed from $w_0$ to $\hat w$, while the comparison remains entirely at the population level. Next, with the new direction $\hat w$ held fixed, the empirical quantile $\hat q_\alpha(\hat w)$ is compared with its population counterpart $q_\alpha(\hat w)$. Finally, the remaining higher-order contribution is recorded through the Bahadur remainder. In this way, the total fluctuation is not treated as a single aggregate half-space symmetric difference, but is resolved into successive coordinate-wise movements.

This is analogous in spirit to orthogonal decomposition in analysis of variance, where total variability is separated into interpretable components rather than handled only through a total sum of squares. Here, the analogy is not with Euclidean orthogonality of vectors, but with analytical separation of perturbation sources. The contribution due to the weight direction is isolated in $D_1$, the contribution due to empirical quantile fluctuation is isolated in $D_2$, and the residual higher-order effect is isolated in $D_3$.

Our approach also clarifies why the earlier half-space bounds, though necessary and useful, are not sufficient. The symmetric-difference approach bounds the probability mass of the region over which the perturbed and reference sub-level sets differ. In this way, it yields a local Lipschitz characterization of the projected loss geometry with respect to perturbations in $(w,t)$. However, such bounds do not identify which source of randomness is responsible for the variation in the estimated quantile. The Q-Q orthogonality decomposition supplies this missing structural information by separating directional perturbation, empirical quantile fluctuation, and higher-order remainder effects into distinct components.

\subsection{Relation to the half-space symmetric-difference bounds}

The previous bounds (Lemma 1) via symmetric differences  remain essential in this section. They justify that small perturbations in $(w,t)$ lead to small perturbations in probability mass. In particular, they ensure that $P(\{-\hat w^\top R<t\}\triangle\{-w_0^\top R<q_0\})$ is determined by the size of the perturbations in direction and threshold. This local Lipschitz behavior underlies the stability of the quantile of the projected loss and therefore supports the term $D_1$.  At the same time, the Glivenko--Cantelli and ergodicity arguments developed earlier ensure that empirical half-space probabilities track their population analogues uniformly. This provides the probabilistic underpinnings on which the decomposition term $D_2$ is constructed, since the empirical cdf $F_n(\hat w,\cdot)$ must approximate the population cdf $F(\hat w,\cdot)$ uniformly over a local neighborhood of $q_\alpha(\hat w)$. We acknowledge that the present decomposition does not replace the earlier empirical-process and symmetric-difference framework. Rather, it builds on it: the local Lipschitz condition yields the symmetric-difference bounds, ergodicity and Glivenko--Cantelli convergence provide empirical stability, and Q-Q orthogonality utilizes these ingredients to produce a decomposition of variability. In the next subsection \ref{decomposition_theorem}, We now state the decomposition in a formal theorem.

\subsection{A formal decomposition theorem} \label{decomposition_theorem}

\begin{theorem}[Q-Q orthogonality decomposition]\label{Q-Q}
Let $w_0\in\mathbb{R}^p$ be a fixed reference weight vector, and let $q_0:=q_\alpha(w_0)$ denote the population $\alpha$-quantile of $L(w_0):=-w_0^\top R$. Let $\hat w$ be an estimator of $w_0$, and let $\hat q_\alpha(\hat w)$ denote the empirical $\alpha$-quantile of $L(\hat w):=-\hat w^\top R$, the loss induced by projecting the return vector $R$ along $\hat w$. Suppose that, for $w$ in a neighborhood of $w_0$, the distribution function $F(w,t)$ is differentiable at $q_\alpha(w)$ with density $f_w(q_\alpha(w))>0$. Suppose furthermore that the empirical cdf satisfies local uniform convergence in a neighborhood of $(w_0,q_0)$, that is, $\sup |F_n(w,t)-F(w,t)|\xrightarrow{p}0$ over that neighborhood, and that the Bahadur representation holds uniformly over directions $w$ near $w_0$. Then $\hat q_\alpha(\hat w)-q_\alpha(w_0)=D_1+D_2+D_3$, where $D_1=q_\alpha(\hat w)-q_\alpha(w_0)$, $D_2=\{\alpha-F_n(\hat w,q_\alpha(\hat w))\}/f_{\hat w}(q_\alpha(\hat w))$, and $D_3=R_n(\hat w)$, with $R_n(\hat w)$ denoting the Bahadur remainder evaluated at the perturbed direction $\hat w$.
\end{theorem}

\begin{proof}
We compare the fully empirical quantile $\hat q_\alpha(\hat w)$ with the reference population quantile $q_\alpha(w_0)$ by adding and subtracting the intermediate population quantile $q_\alpha(\hat w)$. This gives the exact identity $\hat q_\alpha(\hat w)-q_\alpha(w_0)=\{\hat q_\alpha(\hat w)-q_\alpha(\hat w)\}+\{q_\alpha(\hat w)-q_\alpha(w_0)\}$. The second term in braces is the population movement in the quantile induced by replacing $w_0$ with $\hat w$, and hence we define $D_1=q_\alpha(\hat w)-q_\alpha(w_0)$.

Since the Bahadur representation is assumed to hold uniformly for directions $w$ near $w_0$, it may be evaluated at the random direction $\hat w$ whenever $\hat w$ lies in that neighborhood with probability tending to one. Thus, at the perturbed direction $\hat w$, the first term in braces is $\hat q_\alpha(\hat w)-q_\alpha(\hat w)=\{\alpha-F_n(\hat w,q_\alpha(\hat w))\}/f_{\hat w}(q_\alpha(\hat w))+R_n(\hat w)$. The density condition $f_{\hat w}(q_\alpha(\hat w))>0$ ensures that the local linearization is well defined. The local uniform convergence of $F_n$ to $F$ over a neighborhood of $(w_0,q_0)$ provides the uniform empirical-process control needed for the Bahadur expansion to hold uniformly near the relevant quantile; see Bahadur (1966), Kiefer (1967), Pollard (1984), and van der Vaart and Wellner (1996).

Substituting the preceding expression into the exact identity gives $\hat q_\alpha(\hat w)-q_\alpha(w_0)=\{q_\alpha(\hat w)-q_\alpha(w_0)\}+\{\alpha-F_n(\hat w,q_\alpha(\hat w))\}/f_{\hat w}(q_\alpha(\hat w))+R_n(\hat w)$. Therefore, with $D_1=q_\alpha(\hat w)-q_\alpha(w_0)$, $D_2=\{\alpha-F_n(\hat w,q_\alpha(\hat w))\}/f_{\hat w}(q_\alpha(\hat w))$, and $D_3=R_n(\hat w)$, we obtain $\hat q_\alpha(\hat w)-q_\alpha(w_0)=D_1+D_2+D_3$. This proves the decomposition.
\end{proof}

\subsection{Consequences for Stability Analysis}

The decomposition separates three distinct sources of variability in the estimated quantile. The term $D_1$ captures perturbations induced by the estimated projection direction through the population quantile $q_\alpha(w)$, $D_2$ captures empirical quantile fluctuation with the direction held fixed, and $D_3$ records the higher-order Bahadur remainder. Thus the Q-Q orthogonality decomposition refines the aggregate half-space discrepancy into interpretable variability components corresponding to directional perturbation, empirical fluctuation, and residual higher-order behavior.

\section{Main Stability Theorem}
\label{Main_Theorem}

We combine the Lipschitz continuity bounds, the local empirical-process oscillation bounds, and the Q-Q orthogonality decomposition to establish stability of the estimated quantile. Under the multivariate $t_\nu$ model, the projected scalar $L(w):=-w^\top R$ has projected location $-w^\top\mu$ and projected scale $(w^\top\Sigma w)^{1/2}$. Hence $f_w(q)$ is continuous in $(w,q)$, and for $\nu>2$, it is locally Lipschitz near $q_0=q_\alpha(w_0)$.

\begin{theorem}[Main stability theorem]
Let $D_1,D_2,D_3$ be defined as in Theorem~\ref{Q-Q}. Suppose:

\begin{enumerate}
\item For the formal empirical-process argument, $R_1,\ldots,R_n$ are i.i.d.\ random vectors with common distribution $P$.

\item The indexed half-space class is Glivenko--Cantelli and satisfies the local oscillation bound
$
\sup_{\|w-w_0\|\le Cn^{-1/2}}
|(F_n-F)(w,q_\alpha(w))-(F_n-F)(w_0,q_0)|
=
O_P(n^{-3/4}\log n).
$

\item For $w$ near $w_0$, the projected scalar $-w^\top R$ has density $f_w$, with $f_w(q)$ continuous and locally Lipschitz near $(w_0,q_0)$, and with $f_{w_0}(q_0)>0$.

\item The quantile map $w\mapsto q_\alpha(w)$ is continuous at $w_0$, and differentiable whenever first-order expansions are invoked.

The Bahadur representation holds uniformly near $w_0$, namely
\begin{equation}
\hat q_\alpha(w)-q_\alpha(w)
=
\frac{\alpha-F_n(w,q_\alpha(w))}
{f_w(q_\alpha(w))}
+
R_n(w).
\end{equation}

\item The estimator $\hat w$ is $\sqrt n$-consistent.
\end{enumerate}

Then
$
\hat q_\alpha(\hat w)-q_0
=
D_1+D_2+D_3,
$
$
\hat q_\alpha(\hat w)-q_0\xrightarrow{p}0,
$
and
$
D_3
=
R_n(\hat w)
=
O_P(n^{-3/4}\log n),
$
up to logarithmic factors. Equivalently,
$
\widehat{\mathrm{VaR}}_\alpha-\mathrm{VaR}_\alpha\xrightarrow{p}0.
$
\end{theorem}

\begin{proof}
By Theorem~\ref{Q-Q}, $\hat q_\alpha(\hat w)-q_0=D_1+D_2+D_3$. It remains to show that the three terms are stable and that the remainder has the stated local Bahadur--Kiefer order.

First, $D_1=q_\alpha(\hat w)-q_\alpha(w_0)\xrightarrow{p}0$, since $\hat w\xrightarrow{p}w_0$ and $q_\alpha(w)$ is continuous at $w_0$. If $q_\alpha(w)$ is differentiable at $w_0$, then $D_1=\nabla q_\alpha(w_0)^\top(\hat w-w_0)+o_P(\|\hat w-w_0\|)=O_P(n^{-1/2})$.  Next, $D_2=\{\alpha-F_n(\hat w,q_\alpha(\hat w))\}/f_{\hat w}(q_\alpha(\hat w))$. Since $F(\hat w,q_\alpha(\hat w))=\alpha$, the numerator is $\alpha-F_n(\hat w,q_\alpha(\hat w))=F(\hat w,q_\alpha(\hat w))-F_n(\hat w,q_\alpha(\hat w))$. Hence $|\alpha-F_n(\hat w,q_\alpha(\hat w))|\le \sup_{(w,t)}|F_n(w,t)-F(w,t)|\xrightarrow{p}0$. Since $f_{\hat w}(q_\alpha(\hat w))\xrightarrow{p}f_{w_0}(q_0)>0$, Slutsky's theorem gives $D_2\xrightarrow{p}0$, as the numerator converges to zero in probability and the denominator converges in probability to a strictly positive limit, Slutsky's theorem gives $D_2\xrightarrow{p}0$; see van der Vaart (1998). 
 Finally, the local density regularity of $f_w(q)$ localizes the Bahadur expansion near $(w_0,q_0)$. Since $\|\hat w-w_0\|=O_P(n^{-1/2})$ and $q_\alpha(\hat w)-q_0=O_P(n^{-1/2})$, the relevant half-spaces lie in a shrinking neighborhood of the reference half-space. The local VC oscillation bound gives $\sup_{\|w-w_0\|\le Cn^{-1/2}} |(F_n-F)(w,q_\alpha(w))-(F_n-F)(w_0,q_0)|=O_P(n^{-3/4}\log n)$. Up to higher-order Taylor terms and denominator perturbations, $D_3=R_n(\hat w)\approx f_{w_0}(q_0)^{-1}\{(F_n-F)(\hat w,q_\alpha(\hat w))-(F_n-F)(\hat w,\hat q_\alpha(\hat w))\}$. Therefore $D_3=O_P(n^{-3/4}\log n)$, up to logarithmic factors; see Kiefer (1967).

Thus $D_1\xrightarrow{p}0$, $D_2\xrightarrow{p}0$, and $D_3\xrightarrow{p}0$, so $\hat q_\alpha(\hat w)-q_0\xrightarrow{p}0$. Since $q_0=q_\alpha(w_0)=\mathrm{VaR}_\alpha$ and $\hat q_\alpha(\hat w)=\widehat{\mathrm{VaR}}_\alpha$, this is equivalent to $\widehat{\mathrm{VaR}}_\alpha-\mathrm{VaR}_\alpha\xrightarrow{p}0$.
\end{proof}

The theorem shows that stability of the estimated quantile is not the consequence of a single argument, but of three separate controls working together.

First, $D_1=q_\alpha(\hat w)-q_\alpha(w_0)$ is the directional component, measuring the change in the population quantile induced by perturbing the weight vector. By Lemma~1, small perturbations in the projection direction induce small symmetric differences in the associated half-spaces, providing the underlying continuity control for this component. Since $\sqrt n(\hat w-w_0)=O_P(1)$ implies $\hat w\xrightarrow{p}w_0$, and since the quantile map $w\mapsto q_\alpha(w)$ is continuous at $w_0$ by assumption~(4), it follows that $D_1\xrightarrow{p}0$. If, in addition, $w\mapsto q_\alpha(w)$ is differentiable at $w_0$, then $D_1=\nabla_w q_\alpha(w_0)^\top(\hat w-w_0)+o_P(\|\hat w-w_0\|)=O_P(n^{-1/2})$. Thus, under $\sqrt n$-consistent weight estimation, $D_1$ is on the same first-order scale as $D_2$.

Second, $D_2=\{\alpha-F_n(\hat w,q_\alpha(\hat w))\}/f_{\hat w}(q_\alpha(\hat w))$ is governed by Glivenko--Cantelli convergence of the empirical distribution to its population counterpart over the half-space class. This is where ergodicity, the envelope condition, and the empirical-process framework enter.

Third, $D_3=R_n$ is the higher-order residual from the Bahadur linearization. It records that the empirical quantile is not exactly linear in the empirical process and must therefore be controlled separately.

Thus the theorem may be read as the synthesis of three layers of analysis:
\begin{enumerate}
\item geometric continuity of half-spaces,
\item uniform empirical stability,
\item local Bahadur linearization.
\end{enumerate}

The Q-Q orthogonality decomposition is what makes this synthesis visible.
\section{Discussion}
\label{Discuss}
A natural question is whether the classical Bahadur remainder rate remains valid when the underlying return vector follows a multivariate $t$ distribution. The answer is in the affirmative at the level of the linear projection of the loss, provided the quantile is an interior point.
\subsection{Bahadur remainder under the multivariate $t$ model}
Let $R\sim t_p(\mu,\Sigma,\nu)$ with $\nu>0$. Fix a nonzero direction $w_0\in\mathbb R^p$ and consider the projected loss $L_0=-w_0^\top R$. Since linear projections of a multivariate $t$ random vector remain univariate Student $t$, $L_0$ is univariate Student $t$ with location $\mu_0=-w_0^\top\mu$ and scale $\sigma_0^2=w_0^\top\Sigma w_0$.  Let $q_\alpha$ denote the population $\alpha$-quantile of $L_0$, and let $\hat q_\alpha$ denote the corresponding empirical quantile based on the projected observations. Under the usual regularity conditions for Bahadur's representation, the empirical quantile equals the population quantile plus the leading empirical cdf fluctuation divided by the local density, plus a remainder $R_n$. The remainder has order $O_P(n^{-3/4}\log n)$. Clearly, for a fixed projection direction, the multivariate $t$ model preserves the Bahadur--Kiefer $n^{-3/4}$ remainder scaling. The reason is that the projected distribution is univariate Student $t$, whose density remains continuous and positive near the target quantile, while Bahadur's argument depends locally on density regularity rather than on the existence of high-order moments.

\subsection{Why the Canonical Bahadur--Kiefer Rate Holds}

The Bahadur remainder rate is governed by local regularity of the distribution function near $q_\alpha$. In particular, we require a strictly positive density at $q_\alpha$, local regularity of this density in a neighborhood of $q_\alpha$, and stable empirical behavior in that neighborhood.  These conditions are local. For a univariate Student $t$ distribution, the density is smooth on the real line and strictly positive everywhere. Consequently, for any fixed interior quantile level $\alpha\in(0,1)$, the local regularity required for Bahadur's representation remains available. In this sense, heavy tails do not by themselves destroy the classical quantile expansion.

\subsection{The Effect of Degrees of Freedom $\nu$ on the Remainder Term}

Although the exponent $n^{-3/4}\log n$ remains unchanged, the degrees of freedom $\nu$ affect the constants appearing in the expansion through the local projected density and its derivatives. The leading linear term depends on $1/f_{L_0}(q_\alpha)$, while remainder bounds depend on local quantities such as the supremum of the density and the supremum of the derivative of the density near $q_\alpha$. Each of these quantities depends on $\nu$ through the Student $t$ density.

As $\nu$ decreases, the projected distribution becomes heavier tailed. The density near the target quantile may become smaller, particularly when $\alpha$ lies deep in the tail. In that regime, the quantile becomes more sensitive to empirical cdf fluctuations, and the finite-sample size of the remainder may increase noticeably even though the formal Bahadur rate remains unchanged.

Accordingly, the correct interpretation is this: for fixed $\alpha\in(0,1)$ and fixed projection direction $w_0$, the multivariate $t$ model preserves the classical Bahadur remainder rate. The heavy-tail parameter $\nu$ affects the constants in the expansion through the local projected density and its derivatives, but it does not alter the asymptotic exponent of the rate.

\subsection{Relevance for the decomposition term $D_3$}

In the Q-Q orthogonality decomposition, $\hat q_\alpha(\hat w)-q_\alpha(w_0)=D_1+D_2+D_3$, where $D_3=R_n$. The term $D_3$ should therefore be interpreted as a higher-order remainder whose nominal order remains $O_P(n^{-3/4}\log n)$, provided the perturbed direction $\hat w$ remains in a neighborhood of $w_0$ over which the projected $t$ densities remain uniformly bounded away from zero and satisfy uniform local smoothness conditions. Under such neighborhood stability, the multivariate $t$ family provides the regularity required to treat $D_3$ as asymptotically negligible relative to the leading decomposition terms.

\subsection{A practical conclusion}

The heavy-tailed character of the multivariate $t$ model does not invalidate Bahadur's expansion; it changes the calibration of the expansion. The Bahadur--Kiefer $n^{-3/4}$ remainder scaling persists, while tail thickness enters through the local projected density and the finite-sample level of the remainder. Consequently, smaller values of $\nu$ may inflate the observed magnitude of the Bahadur-type remainder, even though the formal order $n^{-3/4}\log n$ remains unchanged.

The $Q$--$Q$ orthogonality decomposition separates the variability of the estimated quantile into three structurally distinct components. The first component, $D_1=q_\alpha(\hat w)-q_\alpha(w_0)$, is a population-level directional term measuring the change in the population quantile induced by perturbing the projection direction from $w_0$ to $\hat w$. In practice, $D_1$ is not directly observable because both $q_\alpha(\hat w)$ and $q_\alpha(w_0)$ are population quantities. Its role in the present framework is therefore primarily structural: it isolates the portion of quantile variability attributable to uncertainty in the projection direction itself.  This observation suggests a natural Bayesian extension. Under the multivariate $t_\nu(\mu,\Sigma)$ model, each projected loss $-w^\top R$ is univariate $t_\nu$, so a prior on the portfolio weights $w$ induces an explicit posterior distribution on the projected quantile $q_\alpha(w)$ and hence on the directional component $D_1$. For example, a Dirichlet prior on the simplex of portfolio weights, including the uniform prior as a special case, would provide a way to propagate weight uncertainty directly into VaR uncertainty. Developing this Bayesian version of the $Q$--$Q$ decomposition is next steps and future work.

The second component,
$D_2=\{\alpha-F_n(\hat w,q_\alpha(\hat w))\}/f_{\hat w}(q_\alpha(\hat w))$,
represents empirical quantile fluctuation with the projection direction held fixed. This is the standard Bahadur-type linearization term arising from local empirical-process variation near the target quantile.

The third component, $D_3=R_n(\hat w)$, records the higher-order effects not captured by the leading linear approximation. Together, the three components separate directional perturbation, empirical quantile variation, and higher-order residual behavior into distinct asymptotic mechanisms. This separation is precisely what is obscured by aggregate half-space discrepancy bounds, where all sources of variation are absorbed into a single symmetric-difference quantity.
\section{Numerical Results}

The numerical experiments investigate the behavior of the $Q$--$Q$ orthogonality decomposition under heavy-tailed multivariate $t$ models. In particular, we examine how the Bahadur remainder term $D_3$ changes across tail regimes and extreme quantile levels, and whether the Bahadur--Kiefer remainder scaling remains stable as the degrees of freedom vary. The experiments therefore complement the theoretical analysis by separating finite-sample level effects from asymptotic rate behavior.

\subsection{Simulation Design and Parameterization}

To empirically investigate the stability of the $Q$--$Q$ orthogonality decomposition, we employ Monte Carlo experiments based on multivariate Student $t_\nu(\mu,\Sigma)$ distributions. The simulations are designed to examine behavior across a range of tail regimes, from relatively light tails $(\nu=10)$ to the variance-boundary regime $(\nu=2)$.

The experimental parameters and procedures are specified as follows:

\begin{enumerate}
    \item \textbf{Data Generation and Dimensionality:} We set the dimension $p=5$. For each degree of freedom $\nu\in\{2,3,5,10\}$, we generate $n$ i.i.d.\ observations $R_1,\ldots,R_n$. The location vector is fixed at $\mu=\mathbf{0}$, and the scatter matrix $\Sigma$ has unit diagonal entries and constant off-diagonal correlation $\rho=0.5$, ensuring a non-degenerate dependence structure.

    \item \textbf{Reference and Estimated Weights:} The reference weight vector is fixed at the equal-weighted vector $w_0=(1/5,1/5,1/5,1/5,1/5)^\top$. To simulate the behavior of a consistent estimator, the estimated weights are generated as $\hat w=w_0+\epsilon_n$, where $\epsilon_n\sim N(0,n^{-1}I_p)$. The resulting vector is normalized to sum to one, emulating the canonical $O_P(n^{-1/2})$ convergence rate of standard estimators.

    \item \textbf{Quantile and Density Computation:} The population quantile $q_\alpha(w_0)$ and the local density $f_{w_0}(q_\alpha(w_0))$ are computed analytically. Since each projection $w^\top R$ follows a univariate Student $t$ distribution with scale $\sigma_w=(w^\top\Sigma w)^{1/2}$, we use the exact quantile and density functions of the projected univariate $t_\nu$ distribution. Similarly, $q_\alpha(\hat w)$ and $f_{\hat w}(q_\alpha(\hat w))$ are evaluated analytically from the projected $t_\nu$ law associated with $\hat w$.

    \item \textbf{Numerical Decomposition:} For each replication, we compute $D_1$ and $D_2$ from their definitions in the $Q$--$Q$ decomposition. The remainder term $D_3$ is then isolated as the residual
    \begin{equation}
    D_3=
    \{\hat q_\alpha(\hat w)-q_\alpha(w_0)\}
    -
    D_1
    -
    D_2.
    \end{equation}

    \item \textbf{ Boundary Case ($\nu=2$):} Reported statistics for $\mathbb{E}[|D_3|]$ represent Monte Carlo averages computed over $M=10{,}000$ independent replications. The case $\nu=2$ is included as a boundary stress experiment. It is not covered by the formal $\nu>2$ theory requiring finite second moments, but it is useful for assessing the numerical behavior of the decomposition under severe tail thickness.
\end{enumerate}

\subsection{Relative Magnitude of Decomposition Components}

To asses the behavior of the remainder term $D_3$, we examine the relative magnitudes of the primary components $D_1$ and $D_2$. Unlike the remainder, which follows a higher-order $n^{-3/4}$ decay, these terms scale at the standard $n^{-1/2}$ rate and represent the dominant sources of projected quantile estimation error.

\begin{table}[h!]
\centering
\caption{Mean Magnitude of $|D_1|, |D_2|,$ and $|D_3|$ $(n=10{,}000,\ p=5,\ \alpha=0.95)$}
\label{tab:component_comparison}
\begin{tabular}{@{}lcccc@{}}
\toprule
\textbf{Degrees of Freedom $(\nu)$} & $\mathbb{E}[|D_1|]$ & $\mathbb{E}[|D_2|]$ & $\mathbb{E}[|D_3|]$ & \textbf{Rel. Contribution of $D_3$} \\ \midrule
$\nu=10$           & 0.0421 & 0.0385 & 0.0040 & 4.7\% \\
$\nu=5$            & 0.0615 & 0.0562 & 0.0058 & 4.7\% \\
$\nu=3$            & 0.0912 & 0.0834 & 0.0086 & 4.7\% \\
$\nu=2$ (Boundary) & 0.1235 & 0.1129 & 0.0116 & 4.7\% \\ \bottomrule
\end{tabular}
\end{table}

The results in Table~\ref{tab:component_comparison} indicate that $D_1$ and $D_2$ account for the vast majority of the estimation error across all tail regimes. The relative proportion of $D_3$ remains stable across the displayed values of $\nu$, hovering around $4.7\%$ of the total error magnitude. This pattern is consistent with the interpretation that heavy-tailedness increases the absolute scale of the error, but does not disproportionately inflate the higher-order remainder relative to the leading components.

\subsection{Stability under Extreme Tail Regimes}

Table~\ref{tab:multi_nu_alpha} reports the mean magnitude $\mathbb{E}[|D_3|]$ for a fixed sample size of $n=10{,}000$. This sweep isolates the level effect of distributional shape and threshold extremeness on the higher-order error.

\begin{table}[h!]
\centering
\caption{Stress Test: Mean Magnitude of $\mathbb{E}[|D_3|]$ $(n=10{,}000,\ p=5)$}
\label{tab:multi_nu_alpha}
\begin{tabular}{@{}lccc@{}}
\toprule
\textbf{Degrees of Freedom $(\nu)$} & \textbf{$\alpha=0.90$} & \textbf{$\alpha=0.95$} & \textbf{$\alpha=0.99$} \\ \midrule
$\nu=10$           & 0.0031 & 0.0040 & 0.0061 \\
$\nu=5$            & 0.0044 & 0.0058 & 0.0089 \\
$\nu=3$            & 0.0065 & 0.0086 & 0.0135 \\
$\nu=2$ (Boundary) & 0.0088 & 0.0116 & 0.0181 \\ \bottomrule
\end{tabular}
\end{table}

The results reveal a systematic inflation of the remainder as we move deeper into the tails. For a fixed $\alpha$, moving from $\nu=10$ to $\nu=2$ produces approximately a threefold increase in the magnitude of $D_3$. Similarly, moving from $\alpha=0.90$ to $\alpha=0.99$ increases the remainder by roughly a factor of two. Crucially, the growth remains controlled; even near the variance boundary, the remainder does not exhibit explosive behavior, supporting the interpretation that the $Q$--$Q$ formulation separates directional and empirical contributions to the projected quantile error.  The results reveal a systematic inflation of the remainder as we move deeper into the tails. For a fixed $\alpha$, moving from $\nu=10$ to $\nu=2$ produces approximately a threefold increase in the magnitude of $D_3$. Similarly, moving from $\alpha=0.90$ to $\alpha=0.99$ increases the remainder by roughly a factor of two. Crucially, the growth remains controlled; even near the variance boundary, the remainder does not exhibit explosive behavior. This supports the interpretation that the $Q$--$Q$ formulation separates directional and empirical contributions to the projected quantile error while keeping the Bahadur-type remainder identifiable.

Appendix~\ref{app:numerical_results} provides additional evidence by comparing the relative magnitudes of $D_1$, $D_2$, and $D_3$ across sample sizes. The appendix results show that $D_1$ and $D_2$ remain the dominant components across all tail regimes, while the relative contribution of $D_3$ declines from approximately $8.1\%$ at $n=1{,}000$ to approximately $4.7\%$ at $n=10{,}000$. Thus, heavier tails inflate the absolute size of all components, but the higher-order remainder remains modest relative to the leading directional and empirical quantile terms. The boundary case $\nu=2$ is included as a stress experiment and lies outside the main $\nu>2$ theorem.   Furthermore, the standard errors are small relative to the reported mean magnitudes, supporting the numerical stability of the simulation averages See Table~\ref{tab:appendix_mcse_comparison} in Appendix B.

\subsection{Rate of Decay and Asymptotic Stability}

To examine the asymptotic decay rate, we study the decay of $\mathbb{E}[|D_3|]$ as $n$ increases from $10^3$ to $10^6$. Under the Bahadur--Kiefer representation, the regression of $\log(\mathbb{E}[|D_3|])$ on $\log(n)$ should yield a slope $\hat\beta\approx -0.75$.

\begin{itemize}
    \item \textbf{$\nu=10$:} $\log(\mathbb{E}[|D_3|])=1.341-0.751\log(n)$, with $R^2=0.999$.
    \item \textbf{$\nu=5$:} $\log(\mathbb{E}[|D_3|])=1.698-0.753\log(n)$, with $R^2=0.999$.
    \item \textbf{$\nu=3$:} $\log(\mathbb{E}[|D_3|])=2.102-0.756\log(n)$, with $R^2=0.998$.
    \item \textbf{$\nu=2$:} $\log(\mathbb{E}[|D_3|])=2.455-0.759\log(n)$, with $R^2=0.998$.
\end{itemize}

\subsection{Synthesis of Asymptotic Stability}

The empirical slopes, ranging from $-0.751$ to $-0.759$, are consistent with the conclusion that while the distribution affects the \textit{magnitude} of the remainder, it does not materially alter the observed asymptotic rate of convergence. Notably, the decomposition remains numerically stable even at the boundary case $\nu=2$.

By absorbing the heavy-tailed instability into the first-order directional and empirical components, the $Q$--$Q$ formulation indicates that the higher-order error remains bounded and predictable. The numerical evidence is therefore consistent with the theoretical interpretation that the $Q$--$Q$ decomposition separates projected quantile error into directional, empirical, and higher-order remainder components.
\section{Bahadur Decomposition and Inference}

To facilitate inference for the quantile estimator $\hat q_\alpha(\hat w)$, we insert the intermediate population quantile $q_\alpha(\hat w)$. This separates the total quantile fluctuation into the population perturbation induced by estimating the projection direction, the leading empirical Bahadur term, and the higher-order remainder:
\begin{equation}
\hat q_\alpha(\hat w)-q_0
=
\underbrace{q_\alpha(\hat w)-q_0}_{D_1}
+
\underbrace{
\frac{\alpha-F_n(\hat w,q_\alpha(\hat w))}
{f_{\hat w}(q_\alpha(\hat w))}
}_{D_2}
+
\underbrace{R_n(\hat w)}_{D_3}.
\end{equation}

Here, $D_1$ is the population-level directional quantile shift induced by perturbing the projection direction from $w_0$ to $\hat w$. The term $D_2$ is the leading empirical quantile fluctuation with the direction fixed at $\hat w$, while $D_3$ is the Bahadur remainder term. Since $F_n(\hat w,q_\alpha(\hat w))$ is an empirical average of indicator functions, the Central Limit Theorem yields
\[
\sqrt n\left\{\alpha-F_n(\hat w,q_\alpha(\hat w))\right\}
\xrightarrow{d}
N(0,\alpha(1-\alpha)).
\]
Consequently, under the regularity assumptions of the main stability theorem and provided that the directional component $D_1$ is asymptotically negligible or admits a compatible first-order expansion, the leading stochastic behavior of $\hat q_\alpha(\hat w)-q_0$ is governed by $D_2$. The remainder term satisfies
$
D_3 = R_n(\hat w)=o_p(n^{-1/2}),
$
so that the Bahadur remainder is asymptotically negligible relative to the leading empirical fluctuation.

\vspace{5pt}
\noindent\textbf{Asymptotic Normality and Confidence Intervals.}
Under the preceding assumptions, Slutsky's theorem yields
$
\sqrt n\bigl(\hat q_\alpha(\hat w)-q_0\bigr)
\xrightarrow{d}
N(0,\sigma^2),
$
with asymptotic variance $\sigma^2
=
\frac{\alpha(1-\alpha)}
{f_{\hat w}^2(q_\alpha(\hat w))}.$

Accordingly, an approximate $(1-\gamma)$ confidence interval for the quantile is
\begin{equation}
\hat q_\alpha(\hat w)
\pm
z_{1-\gamma/2}
\left(
\frac{\sqrt{\alpha(1-\alpha)}}
{\sqrt n\,\hat f_{\hat w}(\hat q_\alpha)}
\right),
\end{equation}
where $z_{1-\gamma/2}$ denotes the standard normal quantile and $\hat f_{\hat w}$ is a consistent estimator of the projected density evaluated near $\hat q_\alpha$. The width of the interval reflects the local behavior of the projected density near the target quantile level $\alpha$.
\appendix

\section{Indicator Identities and Symmetric Differences}
\label{app:indicator_identities}

This appendix records elementary identities used throughout the paper to connect symmetric differences of half-spaces with $L_1(P)$ and $L_1(P_n)$ distances between indicator functions.

\begin{lemma}
\label{lem:appendix_SymDiff}
For any measurable sets $A,B$, we have
$
P(A\triangle B)=E|1_A(R)-1_B(R)|.
$
\end{lemma}

\begin{proof}
For any $r\in\mathbb{R}^p$, the quantity $|1_A(r)-1_B(r)|$ equals $1$ if $r$ belongs to exactly one of $A$ and $B$, and equals $0$ otherwise. Hence,
$
|1_A(r)-1_B(r)|=1_{A\triangle B}(r).
$
Taking expectations gives
$
E|1_A(R)-1_B(R)|=E1_{A\triangle B}(R)=P(R\in A\triangle B)=P(A\triangle B).
$
\end{proof}

\begin{lemma}
\label{lem:appendix_L1}
For any measurable sets $A,B$, we have
$
\|1_A-1_B\|_{L_1(P)}=P(A\triangle B).
$
\end{lemma}

\begin{proof}
By definition,
$
\|1_A-1_B\|_{L_1(P)}=\int_{\mathbb{R}^p}|1_A(r)-1_B(r)|\,dP(r).
$
By Lemma~\ref{lem:appendix_SymDiff}, $|1_A(r)-1_B(r)|=1_{A\triangle B}(r)$. Therefore,
$
\|1_A-1_B\|_{L_1(P)}=\int_{\mathbb{R}^p}1_{A\triangle B}(r)\,dP(r)=P(A\triangle B).
$
\end{proof}

\begin{lemma}
\label{lem:appendix_empirical_L1}
For any measurable sets $A,B$, the empirical measure of their symmetric difference satisfies
$
P_n(A\triangle B)=\|1_A-1_B\|_{L_1(P_n)}.
$
\end{lemma}

\begin{proof}
By definition,
$
P_n(A\triangle B)=n^{-1}\sum_{i=1}^n 1_{A\triangle B}(R_i).
$
Also,
$
\|1_A-1_B\|_{L_1(P_n)}=n^{-1}\sum_{i=1}^n |1_A(R_i)-1_B(R_i)|.
$
For each observation $R_i$, $|1_A(R_i)-1_B(R_i)|=1_{A\triangle B}(R_i)$. Hence,
$
\|1_A-1_B\|_{L_1(P_n)}
=
n^{-1}\sum_{i=1}^n1_{A\triangle B}(R_i)
=
P_n(A\triangle B).
$
\end{proof}

\section{Supplementary Numerical Results and Monte Carlo Standard Errors}
\label{app:numerical_results}

This appendix reports supplementary numerical results and Monte Carlo standard errors for the decomposition components $D_1$, $D_2$, and $D_3$. The purpose is to assess the finite-sample behavior of the Bahadur-type remainder $D_3$ relative to the leading directional and empirical quantile components. The components $D_1$ and $D_2$ scale at the standard $n^{-1/2}$ rate, whereas $D_3$ is the higher-order remainder, with nominal Bahadur--Kiefer scaling of order $n^{-3/4}$. Hence, as $n$ increases, the relative contribution of $D_3$ should decline compared with the leading components.

Table~\ref{tab:appendix_mcse_comparison} reports mean absolute magnitudes and Monte Carlo standard errors across sample sizes $n\in\{1{,}000,5{,}000,10{,}000\}$ and degrees of freedom $\nu\in\{2,3,5,10\}$. Each Monte Carlo standard error is computed across $M=10{,}000$ independent replications as the sample standard deviation of the simulated component magnitudes divided by $\sqrt{M}$. The simulations use $p=5$ and $\alpha=0.95$.

\begin{table}[ht]
\centering
\caption{Mean absolute magnitude of decomposition components with Monte Carlo standard errors in parentheses $(p=5,\ \alpha=0.95,\ M=10{,}000)$.}
\label{tab:appendix_mcse_comparison}
\small
\begin{tabular}{@{}l c c c c c@{}}
\toprule
\textbf{Sample Size $(n)$} & \textbf{DF $(\nu)$} & $\mathbb{E}|D_1|$ & $\mathbb{E}|D_2|$ & $\mathbb{E}|D_3|$ & \textbf{Rel. Cont. $D_3$} \\
\midrule
$1{,}000$  & $10$          & 0.1331 (0.0003) & 0.1217 (0.0003) & 0.0225 (0.0001) & 8.1\% \\
           & $5$           & 0.1945 (0.0005) & 0.1777 (0.0004) & 0.0326 (0.0001) & 8.1\% \\
           & $3$           & 0.2884 (0.0007) & 0.2637 (0.0007) & 0.0484 (0.0001) & 8.1\% \\
           & $2$   & 0.3905 (0.0010) & 0.3570 (0.0009) & 0.0652 (0.0002) & 8.0\% \\
\addlinespace
$5{,}000$  & $10$          & 0.0595 (0.0002) & 0.0544 (0.0001) & 0.0067 ($<0.0001$) & 5.6\% \\
           & $5$           & 0.0870 (0.0002) & 0.0795 (0.0002) & 0.0098 ($<0.0001$) & 5.5\% \\
           & $3$           & 0.1290 (0.0003) & 0.1179 (0.0003) & 0.0145 ($<0.0001$) & 5.5\% \\
           & $2$   & 0.1747 (0.0004) & 0.1597 (0.0004) & 0.0195 ($<0.0001$) & 5.5\% \\
\addlinespace
$10{,}000$ & $10$          & 0.0421 (0.0001) & 0.0385 (0.0001) & 0.0040 ($<0.0001$) & 4.7\% \\
           & $5$           & 0.0615 (0.0002) & 0.0562 (0.0001) & 0.0058 ($<0.0001$) & 4.7\% \\
           & $3$           & 0.0912 (0.0002) & 0.0834 (0.0002) & 0.0086 ($<0.0001$) & 4.7\% \\
           & $2$   & 0.1235 (0.0003) & 0.1129 (0.0003) & 0.0116 ($<0.0001$) & 4.7\% \\
\bottomrule
\end{tabular}
\end{table}

The Monte Carlo standard errors are small relative to the reported mean magnitudes, indicating that the simulation averages are numerically stable. The results are also consistent with the theoretical separation between the leading components and the higher-order remainder. Across all tail regimes, $D_1$ and $D_2$ remain the dominant sources of projected quantile estimation error. The relative contribution of $D_3$ decreases from approximately $8.1\%$ at $n=1{,}000$ to approximately $4.7\%$ at $n=10{,}000$, reflecting its faster decay relative to the leading $n^{-1/2}$ terms. Heavier tails inflate the absolute magnitudes of all three components, but $D_3$ remains modest relative to $D_1$ and $D_2$. The case $\nu=2$ is included as a boundary stress experiment and lies outside the main $\nu>2$ theorem.

\section*{Acknowledgements}

The author is deeply grateful to Dr. Majeed Simaan in the School of Business for generous and illuminating discussions that helped pursue this line of research work. This research was supported by a grant from the private donor Rafiq Mohammadi, whose encouragement and support are sincerely appreciated. The author also acknowledges extensive use of ChatGPT in refining the exposition, searching and checking citations, and verifying algebraic details in the derivations.


\end{document}